\documentclass[11pt]{article}

\usepackage[a4paper,margin=1in]{geometry}
\usepackage{amsmath,amssymb,amsthm,mathtools,bm}
\usepackage{booktabs,longtable,array}
\usepackage{enumitem}
\usepackage{xcolor}
\usepackage{hyperref}
\usepackage{algorithm}
\usepackage{algpseudocode}
\usepackage{float}
\usepackage{microtype}
\usepackage{ulem}
\usepackage{graphicx}
\usepackage{subcaption}

\usepackage{caption}
\usepackage{hyperref}
\hypersetup{
  colorlinks=true,
  linkcolor=blue!60!black,
  citecolor=blue!60!black,
  urlcolor=blue!60!black
}

\newcommand{\NA}{--}

\newtheorem{assumption}{Assumption}[section]
\newtheorem{theorem}{Theorem}[section]
\newtheorem{proposition}[theorem]{Proposition}
\newtheorem{lemma}[theorem]{Lemma}

\newtheoremstyle{remarkitalic}
  {3pt}        
  {3pt}        
  {\itshape}   
  {}           
  {\bfseries}  
  {.}          
  {.5em}       
  {}           
\theoremstyle{remarkitalic}
\newtheorem{definition}[theorem]{Definition}
\theoremstyle{remarkitalic}
\newtheorem{remark}[theorem]{Remark}

\newcommand{\R}{\mathbb R}

\newcommand{\E}{\mathbb E}
\newcommand{\Prob}{\mathbb P}

\newcommand{\Sspace}{\mathcal S}
\newcommand{\dd}{\,\mathrm d}

\newcommand{\norm}[1]{\left\lVert #1 \right\rVert}
\newcommand{\abs}[1]{\left\lvert #1\right\rvert}
\newcommand{\Tr}{\operatorname{Tr}}

\usepackage{verbatim} 
\title{Continuous-Time Reinforcement Learning for Controlled Hawkes Jump-Diffusions}
\author{Tomasz R. \textsc{Bielecki}\footnote{Hinsdale, Illinois, USA, \url{tuslecki@gmail.com}}\and Thibaut \textsc{Mastrolia}\footnote{UC Berkeley, Department of Industrial Engineering and Operations Research, USA, \url{mastrolia@berkeley.edu}}\and Haoze \textsc{Yan}\footnote{UC Berkeley, Department of Industrial Engineering and Operations Research, USA, \url{haoze.yan@berkeley.edu}} }
\date{\today}

\begin{document}
\maketitle

\begin{abstract}
We study stochastic control of multivariate Hawkes-driven stochastic differential equations with machine learning algorithms in a non-Markovian setting. Due to the path dependence of the memory of the Hawkes intensity, this problem does not fall within classical stochastic control theory outside particular Markovian kernels. We first develop a finite-dimensional Markovianization procedure and algorithm to approximate multivariate Hawkes processes with mixtures of exponential kernels. We prove the convergence of the Markovianized approximation of the Hawkes process, its intensity, and the value of the problem to the original non-Markovian processes and the value of the primal problem. We then formulate continuous-time deterministic policy gradient learning on the Markovianized approximation of the problem, called Hawkes-CT DDPG. We propose a model-free algorithm to solve the non-Markovian Hawkes-driven optimization by observing only the event times of the process, the realization of the solution to the SDE, and a chosen set of decay filters, while the Hawkes kernel coefficients remain unknown. We compare our continuous time reinforcement learning Hawkes-CT DDPG method with discrete time reinforcement learning techniques under three different types of kernels: simple exponential, Erlang, and power-law kernels.
\end{abstract}

\section{Introduction}

Cybersecurity has become a major operational concern for large and highly interconnected organizations, where a single security incident can rapidly disrupt essential services and expose sensitive information. A recent illustration is the May 2026 cybersecurity incident affecting Canvas, a cloud-based learning management system used by thousands of schools and universities. The breach led for example UC Berkeley to temporarily restrict access to the platform and potentially exposed information including names, email addresses, student identification numbers, and messages exchanged through Canvas, highlighting the vulnerability of organizations to attacks propagated through third-party digital infrastructure. Such events also illustrate an important feature of cyber risk: attacks rarely occur as isolated and independent events, but may instead arrive in clusters, with previous incidents increasing the likelihood of subsequent malicious activity. Hawkes processes provide a natural probabilistic framework for representing this self-exciting behavior \cite{baldwin2017contagion}. At the same time, defenders must dynamically decide how to allocate limited security resources as the threat environment evolves. This motivates the combination of Hawkes-process models of cyber-event arrivals with reinforcement learning, which provides a flexible framework for learning adaptive defense strategies in an uncertain and dynamically changing environment. The objective of this study is twofold: study the controlled system of a Hawkes-driven system for general self-excitation processes and develop reinforcement techniques in continuous time to minimize costs related to cascade incidents in an unknown environment.  
\subsection{Hawkes processes: theory, applications, and control}

In his pioneering work \cite{hawkes1971spectra}, Alan G. Hawkes introduced the theory of self-exciting point processes, in which the occurrence of an event raises the likelihood of subsequent events; the intensity then decays over time before returning to a baseline level. Since this seminal work, the theory and applications of it have known a growing interest. We refer to for example some recent studies focusing on multivariate processes  \cite{Bielecki2022, Bielecki2022a}, parameters estimation for linear Hawkes process with general kernel \cite{cartea2021gradient}, limit and convergence theorems and approximation results \cite{jaisson2015limit,khabou2024normal,coutin2025quantification,pace2024convergence}, expansion formulas and Malliavin calculus \cite{hillairet2025explicit,hillairet2025multivariate} or extensions to rough diffusions \cite{jaisson2016rough,el2019characteristic,bondi2024rough}. These processes have found a remarkably wide range of applications, including earthquake modeling \cite{van2010connecting,kwon2023flexible}, finance and microstructure \cite{embrechts2011multivariate,horst2022microstructure}, market order flow modeling \cite{jusselin2020no,jusselin2021optimal,graf2026learning,gennaro2026signature}, see also the review papers \cite{bacry2015hawkes,hawkes2018hawkes}, credit risk \cite{errais2010affine}, epidemiology \cite{rizoiu2018sir}, soccer timing of threat events and ball touches of players \cite{baouan2025crediting}, more recently, cyber risk \cite{baldwin2017contagion,dubois2022cyber,bessy2021multivariate, callegaro2025stochastic,hillairet2023expansion} or cascade accidents in power plant investment \cite{agostino2026approximation}. This last application leads naturally to the question of controlling Hawkes processes, in which a hacker and a defender each seek to optimally manage their losses over the course of an attack episode.\vspace{0.3em}

Numerical schemes for simulating Hawkes processes via the so-called thinning method were first investigated in \cite{lewis1976simulation,ogata1981lewis}, and later rigorously formalized through a two-dimensional Poisson measure by Br\'emaud and Massouli\'e \cite{bremaud1996stability}. Numerical methods for generalized Hawkes have also been recently developed in \cite{zhang2009rare}. The emergence of cyber risk has renewed the mathematical community's interest in control of Hawkes processes, inducing new challenges related to the control of the intensity process. In its simplest form, a Hawkes process $N$ is a counting process whose self-exciting intensity $\lambda$ is given by
\[
    \lambda_t = \mu_\infty(t) + \int_0^{t-} \Phi(t-s)\,dN_s,
\]
where $\mu_\infty$ is the baseline intensity and $\Phi$ is a kernel that measures the magnitude of the self-excited jumps. The aim of this paper is to study the control of such dynamics in a multidimensional framework, namely
\[
    \dd X_t
    =
    b(t,X_t,a_t)\,\dd t
    +
    \sigma(t,X_t,a_t)\,\dd W_t
    +
    \sum_{i=1}^m \gamma_i(t,X_{t-},a_t)\,\dd N_t^i,
\]
where the intensity of $N = (N^1,\dots,N^m)$ is denoted by
$ \lambda_t = (\lambda_t^1,\dots,\lambda_t^m)^\top \in \R_+^m,$
and is assumed to be defined componentwise by
\[
    \lambda^i_t
    =
    \mu_i(t,X_{t-},a_t)
    +
    \sum_{j=1}^m
    \int_0^{t-}\Phi_{ij}(t-s)\,\dd N_s^j,
\]
where $\Phi$ is a mutual excitation matrix between each component of $N$.
This kernel lies at the heart of the complexity of Hawkes processes, and developing a self-contained control theory for them remains an open challenge, with only partial preliminary results available so far. When $m=1$ and $\Phi$ takes the exponential form $\Phi(s) = q\,e^{-\beta s}$, known as the exponential kernel, the intensity $\lambda$ admits an explicit expression and, in particular, becomes a Markov process. \cite{Bensoussan2024} exploited this structure to reduce the stochastic control of a Hawkes-driven diffusion with an exponential kernel to a standard Markovian optimization problem with state variables $(N,\lambda)$ later extended in \cite{callegaro2025stochastic,ongarato2026hawkes} to the control of cyber systems, again restricting attention to a single exponential kernel. To the best of our knowledge, the main recent breakthrough is due to \cite{KhabouTalbi2025}, who approximate the value of a stochastic control problem driven by a general-kernel Hawkes process by that of a regularized version driven by a mixture of exponentials. This construction, known as the Markovianization of the generalized Hawkes model, has also been studied for Volterra processes in \cite{abi2019multifactor}.

\subsection{Hawkes CT-DDPG: the big picture}

The first ingredient of our Hawkes CT-DDPG method is to reconstruct the approximate, Markovianized control problem using only the observed realizations of the controlled diffusion together with its event (jump) times, and to store in a replay buffer the new state, composed of the updated diffusion value and its jump components. More precisely, the control process $a$ at time $t$ governs how the existing Hawkes memory is read at that time, while past jumps continue to contribute through a classical Hawkes kernel. The Markovian approximation of a kernel $\Phi$ controlled by $a$ takes the form
\[
    \Phi(\tau,a) \approx \Phi_K(\tau,a) = \sum_{k=1}^K Q_k(a)\,e^{-\beta_k \tau}.
\]
 Note that mixtures of exponentials with possibly negative weights are dense in $L^p$, see \cite{KAMMLER1976384}. We use this fundamental result for the Markovianization of the kernel $\Phi$ above through the matrices $Q_k(a)$ which may have non-positive entries. This may produce negative raw memory values. To avoid that, we define the componentwise positive part to the approximated memory term:
\[
    \lambda_t^K
    =
    \mu(t,X_{t-}^K,a_t)
    +
    \left(\sum_{k=1}^K Q_k(a_t)\,Z_{t-}^{K,k}\right)_+,
\]
where $Z^{K,k}=(Z^{K,k,j})_{j=1}^m$, driven by the jumps of $N^K$, componentwise by
\begin{equation*}
Z_{t-}^{K,k,j}
:=
\int_{(0,t)} e^{-\beta k(t-s)}\,\dd N_s^{K,j},
\qquad j=1,\dots,m.
\end{equation*}  

A key technical point is that the memory variables entering the intensity must be predictable. For this reason, throughout the paper the approximating intensity is evaluated at the left-limit memory $Z_{t-}^{K,k}$. The process $Z_t^{K,k}$ itself is c\`adl\`ag and jumps at event times; this convention keeps the intensity predictable while preserving the Markov property of the lifted process. This is the purpose of our first layer, summarized in Algorithm~\ref{alg:state_update} below.\\

The second ingredient builds on the recent works on continuous time reinforcement learning, investigated in \cite{doya2000reinforcement} and used in stochastic control in \cite{wang2020reinforcement} with a relaxed-control formulation and more recently in \cite{ChengGuoZhang2026} which uses a Continuous-Time Deep Deterministic Policy Gradient (CT-DDPG) method to update both the control (actor) and the value function (critic) with neural networks. The central idea is to define the training losses for the network weights by extracting as much information as possible from the martingale property of stochastic integrals (the so-called generalized moment method), together with an advantage rate designed to learn the Hamiltonian of the corresponding HJB equation, which is the only quantity containing the unknown parameters. This second layer is summarized in Algorithm~\ref{alg:hawkes_ctddpg} below.\\

Finally, the convergence of our algorithm rests on two main contributions: the Markovianization approximation and the error incurred when training the actor and critic networks via the Hawkes CT-DDPG method. The former can be analyzed theoretically, whereas the latter remains poorly understood within current reinforcement learning theory and can only be assessed numerically, by benchmarking our algorithm against an oracle and against analytical solutions, where these exist, as a sanity check.\\

The paper is organized as follows. Section~\ref{sec:model} introduces the controlled Hawkes jump--diffusion model together with the main standing assumptions. Section~\ref{sec:markov_approx} develops the exponential Markov approximation based on mixtures of exponentials, establishes the corresponding convergence theorem, and presents the online Markov-state update algorithm. Section~\ref{sec:ctddpg} formulates the Hawkes CT-DDPG method, including the Bellman equation, the advantage rate, the deterministic policy gradient identity, the martingale characterization, and the implementation algorithm. Finally, Section~5 presents numerical experiments under three kernels:
a single-exponential kernel, an Erlang kernel, and a power-law kernel.
For the first two cases, the exact finite-dimensional Markov
representations allow us to construct known-parameter DGM/HJB oracle
benchmarks. For the power-law case, which has no exact
finite-dimensional Markov representation, we instead use a
DGM/HJB benchmark obtained after approximating the kernel by a finite
mixture of exponentials. We compare Hawkes CT-DDPG with these
model-based benchmarks, a validation-selected static policy, and the
discrete-time reinforcement-learning methods SAC and DDPG, and
we show, in particular, that Hawkes CT-DDPG remains the closest approximator of the oracle in all scenarios and consistently outperforms the other discrete RL methods, thereby extending the findings of \cite{ChengGuoZhang2026} to Hawkes-driven diffusions. 

\section{Mathematical framework and controlled Hawkes jump-diffusions}
\label{sec:model}

In the full paper we set a probability space $(\Omega,\mathcal F,\mathbb P)$ endowed with a $d_W$-dimensional Browninan motion $W$. 
Let $T>0$ be a finite horizon. 
Let $\Pi:=(\Pi^i)_{i\leq m}$ be a family of $m$ independent Poisson random measures on $(0,T]\times\mathbb{R}_+$ each with compensator $ds\,d\theta$. We set
\[
\mathcal{F}_t
:=
\sigma\!\left(
X_0,
W_s,
\Pi^i((0,s]\times B):
0\leq s\leq t,\ i=1,\ldots,m,\
B\in\mathcal{B}(\mathbb{R}_+),\ \operatorname{Leb}(B)<\infty
\right),
\]
and let $\mathbb{F}$ be the usual augmentation of
$\mathbb{F}^0=(\mathcal{F}^{0}_t)_{0\leq t\leq T}$.
We call $\mathbb{F}$ the environmental filtration.

Let $A\subset\R^{d_a}$ be the action space with $d_a>0$ supposed to be compact. An admissible control, also called a policy in the optimization problem,
is an $A$-valued $\mathbb{F}$-predictable process
$
a=(a_t)_{0\leq t\leq T}.$
The random variable $a_t$ is the local action applied at time $t$. We set
\[
\mathcal{A}
:=
\left\{
 a:[0,T]\times\Omega\longrightarrow A:
 a\text{ is }\mathbb{F}\text{-predictable}
\right\}.
\]
\begin{remark}
    The class $\mathcal{A}$ is fixed independently of the kernel
approximation index $K$ defined in the following section. In particular, when the original and
approximating systems are compared later,
the same value $a_t(\omega),; \omega\in \Omega$ is inserted into both systems.
\end{remark} Let $b,\sigma,\gamma$ be the drift, volatility and jump severity coefficients defined for $d_x>0$ by

\[ b:[0,T]\times\R^{d_x}\times A\to\R^{d_x},
    \qquad
    \sigma:[0,T]\times\R^{d_x}\times A\to\R^{d_x\times d_W},
\]
\[
    \gamma_i:[0,T]\times\R^{d_x}\times A\to\R^{d_x},
    \qquad i=1,\dots,m.
\] We now define the controlled Hawkes-driven SDE system. Let $(\Omega,\mathcal F,\mathbb F,\mathbb P)$ support an
$\mathbb R^{d_x}$-valued initial condition $X_0$, a
$d_W$-dimensional Brownian motion $W$, and mutually independent
Poisson random measures $\Pi^1,\ldots,\Pi^m$, independent of
$(X_0,W)$, each with intensity $dt\,d\theta$. Let $a\in\mathcal{A}$. A controlled Hawkes jump--diffusion under $a$
is a strong solution
$
(X^a,N^a,\lambda^a)$
of the coupled Brownian-Poisson measure system
\begin{equation}
\label{eq:poisson-embedded-hawkes-sde}
\left\{
\begin{aligned}
X_t^a
={}&X_0
+\int_0^t b(s,X_{s-}^a,a_s)\,ds
+\int_0^t \sigma(s,X_{s-}^a,a_s)\,dW_s
\\
&\quad
+\sum_{i=1}^m
\int_{(0,t]}\int_0^\infty
\gamma_i(s,X_{s-}^a,a_s)
\mathbf{1}_{\{\theta\leq\lambda_s^{a,i}\}}
\,\Pi^i(ds,d\theta),\; X_0\in \mathbb R^{d_x},
\\[0.25cm]
N_t^{a,i}
={}&
\int_{(0,t]}\int_0^\infty
\mathbf{1}_{\{\theta\leq\lambda_s^{a,i}\}}
\,\Pi^i(ds,d\theta),
\qquad i=1,\ldots,m,
\\[0.25cm]
\lambda_t^{a,i}
={}&
\mu_i(t,X_{t-}^a,a_t)
\\
&\quad
+\sum_{j=1}^m
\int_{(0,t)}\int_0^\infty
\Phi_{ij}(t-s,a_t)
\mathbf{1}_{\{\theta\leq\lambda_s^{a,j}\}}
\,\Pi^j(ds,d\theta),
\qquad i=1,\ldots,m.
\end{aligned}
\right.
\end{equation}
\begin{remark}
  The control inside the kernel $\Phi$ is given by the current action $a_t$, not the historical action $(a_s)_{s\leq t}$. The coefficient $\Phi_{ij}(\tau,a)$ represents the lag-$\tau$ excitation of component $i$ caused by a past jump in component $j$, as read under the current action $a_t$. 
\end{remark}

\noindent Here $X^a$ and $N^a$ are c\`adl\`ag, and $\lambda^a$ is required to be
nonnegative and $\mathbb{F}$-predictable. Equivalently,
\begin{equation}
\label{eq:controlled-hawkes-intensity}
\lambda_t^a
=
\mu(t,X_{t-}^a,a_t)
+
\int_{(0,t)}\Phi(t-s,a_t)\,dN_s^a,
\end{equation}
and the state equation satisfies
\[
dX_t^a
=
b(t,X_{t-}^a,a_t)\,dt
+\sigma(t,X_{t-}^a,a_t)\,dW_t
+\sum_{i=1}^m
\gamma_i(t,X_{t-}^a,a_t)\,dN_t^{a,i}.
\]

For a generic current action $a\in A$, define the current-action memory
readout by
\begin{equation}
\label{eq:current-action-memory}
H_t^a(a)
:=
\int_{(0,t)}\Phi(t-s,a_t)\,dN_s^a.
\end{equation}
Therefore
$\lambda_t^a
=
\mu(t,X_{t-}^a,a_t)+H_t^a(a_t).$
Thus the action used to read the accumulated Hawkes memory at time $t$
is the current action $a_t$, rather than the historical actions applied
when the past events occurred.

\begin{remark}
The strict upper limit $(0,t)$ in
\eqref{eq:controlled-hawkes-intensity} ensures that
$\lambda_t^a$ depends only on accepted events strictly before $t$.
Together with the predictability of $a_t$ and $X_{t-}^a$, this avoids ill-posedness of the equation and any algebraic loop at time $t$.
\end{remark}
We define the compensated Poisson measure by
$
\widetilde\Pi^i(ds,d\theta)
:=
\Pi^i(ds,d\theta)-ds\,d\theta.$
Note that whenever $\int_0^T\lambda_s^{a,i}\,ds<\infty$ almost surely,
\begin{equation}
\label{eq:compensated-hawkes-martingale}
\begin{aligned}
M_t^{a,i}
&:=N_t^{a,i}-\int_0^t\lambda_s^{a,i}\,ds=
\int_{(0,t)}\int_0^\infty
\mathbf{1}_{\{\theta\leq\lambda_s^{a,i}\}}
\,\widetilde\Pi^i(ds,d\theta)
\end{aligned}
\end{equation}
is an $\mathbb{F}$-local martingale. Consequently,
$\int_0^t\lambda_s^{a,i}\,ds$ is the $\mathbb{F}$-predictable
compensator of $N^{a,i}$, and $\lambda^{a,i}$ is its intensity.
\begin{remark}[Environmental and observation filtrations]
\label{rem:environment-observation-filtrations}
For a fixed admissible control $a$, one may define the observation
filtration $\mathbb{F}^{a,\mathrm{obs}}$ as the usual augmentation of
\[
\mathcal{F}^{a,\mathrm{obs}}_t
:=
\sigma\!\left(
X_s^a,N_s^a:0\leq s\leq t
\right).
\]
Then $\mathbb{F}^{a,\mathrm{obs}}\subseteq\mathbb{F}$. This
policy-dependent filtration is useful for describing what is supplied
to the learning algorithm, but it is not used to define the common
admissible class $\mathcal A$ in the approximation theorem.
\end{remark} 

\noindent Throughout, $\|\cdot\|$ denotes the Euclidean norm on
$\mathbb R^{d_x}$ and $\mathbb R^{d_a}$, and
$\|\cdot\|_{\mathrm F}$ denotes the Frobenius norm on
$\mathbb R^{d_x\times d_W}$. For vectors in $\mathbb R^m$,
we use the standard $\ell^1$-norm $\|\cdot\|_1$, while for
$M\in\mathbb R^{m\times m}$,
\[
    \|M\|_{1\to1}
    :=
    \sup_{v\ne0}\frac{\|Mv\|_1}{\|v\|_1}
    =
    \max_{1\le j\le m}\sum_{i=1}^m |M_{ij}|.
\]
This norm is natural for the Hawkes memory because a jump in component
$j$ contributes the column $Me_j$, and
\[
\|M\,\Delta N_s\|_1
\leq
\|M\|_{1\to1}\,\|\Delta N_s\|_1.
\]
\noindent We now set the standing assumption, enforced along this study.

\begin{assumption}[Standing assumptions]
\label{ass:standing}
The following conditions hold.
\begin{enumerate}

    \item[(S1)] The maps $b$, $\sigma$, and $\gamma_i$, $i=1,\ldots,m$, are
    Borel measurable and continuous in the time and action variable. There exists
    $L>0$ such that, uniformly in $(t,u)\in[0,T]\times A$,
    \[
    \begin{aligned}
    &\|b(t,x,u)-b(t,x',u)\|
    +\|\sigma(t,x,u)-\sigma(t,x',u)\|
    \\
    &\qquad
    +\sum_{i=1}^m
    \|\gamma_i(t,x,u)-\gamma_i(t,x',u)\|
    \leq L\|x-x'\|,
    \end{aligned}
    \]
    and
    \[
    \|b(t,x,u)\|+\|\sigma(t,x,u)\|
    +\sum_{i=1}^m\|\gamma_i(t,x,u)\|
    \leq L(1+\|x\|).
    \]

    \item[(S2)] The baseline map $\mu$ is Borel measurable, nonnegative, and
    continuous in the time and action variable. Uniformly in
    $(t,u)\in[0,T]\times A$,
    \[
    \|\mu(t,x,u)-\mu(t,x',u)\|_1
    \leq L\|x-x'\|,
    \qquad
    \|\mu(t,x,u)\|_1\leq L(1+\|x\|).
    \]
    \item[(S3)] The kernel $\Phi$ is Borel measurable and entrywise
    nonnegative. For every $T>0$,
    \[
    \overline\phi_T(\cdot)
    :=
    \sup_{u\in A}\|\Phi(\cdot,u)\|_{1\to1}
    \in L^1([0,T]).
    \]
    Moreover, for every $a\in\mathcal A$, the process
    $t\mapsto H_t^a(a_t)$ defined by \eqref{eq:current-action-memory} admits an
    $\mathbb F$-predictable version.

    \item[(S4)] For every $a\in\mathcal A$, the system
    \eqref{eq:poisson-embedded-hawkes-sde} admits a pathwise unique strong solution $(X^a,N^a,\lambda^a)$ on $[0,T]$.
    For every $p\geq1$, there exists $C_{p,T}<\infty$, independent of
    $a$, such that
    \[
    \sup_{a\in\mathcal A}
    \mathbb E\!\left[
       \sup_{0\leq t\leq T}\|X_t^a\|^p
       +\|N_T^a\|_1^p
       +\left(\int_0^T\|\lambda_t^a\|_1\,dt\right)^p
    \right]
    \leq C_{p,T}.
    \]
    \end{enumerate}
\end{assumption}

\begin{remark}[Scope of the well-posedness assumption]
This assumption isolates strong existence,
pathwise uniqueness, nonexplosion, and uniform moment control from the
Markov-approximation argument. All these properties can be verified under
appropriate Lyapunov and stability conditions for state-dependent
Hawkes intensities, see for example \cite[Proposition 1.4]{KhabouTalbi2025}. The focus of this paper is the finite-dimensional
kernel approximation and its use in continuous-time actor--critic
learning, rather than a complete existence theory for the nonlinear
closed-loop Hawkes system. In particular, these properties are satisfied in the numerical example we investigate below, see Section \ref{sec:numerics}.
\end{remark}

When no confusion can arise, we suppress the superscript $a$ and write
$(X,N,\lambda)$ for the system controlled by a fixed
$a\in\mathcal A$.

\section{Finite-Dimensional Markov Approximation}
\label{sec:markov_approx}


In order to develop an RL method in continuous time extending \cite{ChengGuoZhang2026} to Hawkes process, we first need to reduce the study to a Markovian optimization with its associated HJB equation. The key idea to to approach any integrable kernel potentially memory dependent with a family of exponential kernels, known as a mixture of exponentials, by using a density argument and convergence of value functions associated to this Markovianized version. This process has been previously developed in \cite{abi2019multifactor} for Volterra processes and \cite{KhabouTalbi2025} for Hawkes processes. In our model however, we face a technical challenge since the Markovianized kernel is itself controlled and must approximated the general one unformly with respect to the control. We first set the following assumption to approach any kernel $\Phi\in L^p$ for some $p\geq 1$. Then we turn to the Markovianization of the state process and finally to the convergence of  the value and objective of the Markovianization controlled process and Hawkes diffusion to the general one.  
\subsection{Signed exponential approximation}

We enforce the following assumption in this study.

\begin{assumption}
\label{ass:kernel_approx_strong}
Fix a decay scale $\beta>0$. For each $K\geq 1$, there exists a kernel $\Phi_K$ and a measurable continuous matrix coefficient functions
\[
    Q_k:A\to\R^{m\times m},
    \qquad k=1,\dots,K,
\]
defined by $
    \Phi_K(\tau,a)
    :=
    \sum_{k=1}^K Q_k(a)e^{-\beta k\tau}.
$ such that

\[\delta_K(T)
    :=
    \int_0^T \bar\varepsilon_K(\tau)\,\dd\tau
    \underset{K\to\infty}{\longrightarrow}0,\] where \[\bar\varepsilon_K(\tau)
    :=
    \sup_{a\in A}
    \norm{\Phi_K(\tau,a)-\Phi(\tau,a)}_{1\to 1}.\]
\end{assumption}

\begin{lemma}\label{lemma:bounds}
Suppose Assumption~\ref{ass:standing}(S3) and
Assumption~\ref{ass:kernel_approx_strong} hold.
\begin{itemize}
    \item[(i)]
    Let
    \[
        \bar\phi_K(\tau)
        :=
        \sup_{a\in A}
        \norm{\Phi_K(\tau,a)}_{1\to1}.
    \]
    Then
    \[
        \sup_{K\ge1}
        \int_0^T \bar\phi_K(\tau)\,\dd\tau
        <\infty.
    \]

    \item[(ii)]
    For every $K\ge1$, the scalar Volterra equation
    \begin{equation}\label{eq:volterra}
        R_K(t)
        =
        \bar\phi_K(t)
        +
        \int_0^t
        \bar\phi_K(t-s)R_K(s)\,\dd s
    \end{equation}
    admits a unique solution $R_K\in L^1(0,T)$. This solution is
    nonnegative and satisfies
    \[
        \sup_{K\ge1}
        \int_0^T R_K(t)\,\dd t
        <\infty.
    \]
\end{itemize}
\end{lemma}

\begin{proof}
Let
\[
    \varphi_T(\tau)
    :=
    \sup_{a\in A}
    \norm{\Phi(\tau,a)}_{1\to1}.
\]
For almost every $\tau\in[0,T]$, the reverse triangle inequality and
the elementary inequality
\[
    \left|
        \sup_{a\in A} f(a)-\sup_{a\in A}g(a)
    \right|
    \le
    \sup_{a\in A}|f(a)-g(a)|
\]
give
\begin{align*}
    \left|
        \bar\phi_K(\tau)-\varphi_T(\tau)
    \right|
    &\le
    \sup_{a\in A}
    \left|
        \norm{\Phi_K(\tau,a)}_{1\to1}
        -
        \norm{\Phi(\tau,a)}_{1\to1}
    \right|\\
    &\le
    \sup_{a\in A}
    \norm{
        \Phi_K(\tau,a)-\Phi(\tau,a)
    }_{1\to1}\\
    &=
    \bar\varepsilon_K(\tau).
\end{align*}
Consequently,
\begin{equation}\label{eq:phi-envelope-convergence}
    \norm{
        \bar\phi_K-\varphi_T
    }_{L^1(0,T)}
    \le
    \delta_K(T)
    \longrightarrow 0.
\end{equation}
In particular,
\[
    \int_0^T\bar\phi_K(\tau)\,\dd\tau
    \le
    \int_0^T\varphi_T(\tau)\,\dd\tau
    +
    \delta_K(T).
\]
Since $\varphi_T\in L^1(0,T)$ by
Assumption~\ref{ass:standing}(S3), and the convergent sequence
$(\delta_K(T))_{K\ge1}$ is bounded, we obtain
\[
    \sup_{K\ge1}
    \int_0^T\bar\phi_K(\tau)\,\dd\tau
    <\infty.
\]
This proves part~\textup{(i)}.

\noindent We next prove part~\textup{(ii)}. For $\alpha>0$, define the weighted
$L^1$-norm
\[
    \norm{f}_{1,\alpha}
    :=
    \int_0^T e^{-\alpha t}|f(t)|\,\dd t.
\]
Since
\[
    e^{-\alpha T}\norm{f}_{L^1(0,T)}
    \le
    \norm{f}_{1,\alpha}
    \le
    \norm{f}_{L^1(0,T)},
\]
this norm is equivalent to the usual $L^1(0,T)$-norm.

\noindent For $f,g\in L^1(0,T)$, Fubini's theorem and the change of variable
$u=t-s$ yield
\begin{align}
    \norm{f*g}_{1,\alpha}
    &\le
    \int_0^T\int_0^t
    e^{-\alpha(t-s)}|f(t-s)|
    e^{-\alpha s}|g(s)|
    \,\dd s\,\dd t \notag\\
    &=
    \int_0^T
    e^{-\alpha s}|g(s)|
    \left(
        \int_0^{T-s}
        e^{-\alpha u}|f(u)|\,\dd u
    \right)
    \dd s \notag\\
    &\le
    \norm{f}_{1,\alpha}
    \norm{g}_{1,\alpha}.
    \label{eq:weighted-convolution}
\end{align} Fix $K\ge1$. By part~\textup{(i)},
$\bar\phi_K\in L^1(0,T)$. Moreover, dominated convergence gives
\[
    \norm{\bar\phi_K}_{1,\alpha}
    =
    \int_0^T
    e^{-\alpha t}\bar\phi_K(t)\,\dd t
    \longrightarrow 0
    \qquad\text{as }\alpha\to\infty.
\]
Therefore, one may choose $\alpha_K>0$ such that
\[
    q_K
    :=
    \norm{\bar\phi_K}_{1,\alpha_K}
    <1.
\] Consider the affine map
\[
    \mathcal T_K f
    :=
    \bar\phi_K+\bar\phi_K*f
\]
on the Banach space $L^1(0,T)$ equipped with
$\norm{\cdot}_{1,\alpha_K}$. By
\eqref{eq:weighted-convolution},
\[
    \norm{
        \mathcal T_K f-\mathcal T_K g
    }_{1,\alpha_K}
    \le
    q_K
    \norm{f-g}_{1,\alpha_K}.
\]
Thus $\mathcal T_K$ is a contraction. The Banach fixed-point theorem
therefore gives a unique $R_K\in L^1(0,T)$ satisfying
\eqref{eq:volterra} almost everywhere on $(0,T)$. The solution is nonnegative. Indeed, starting from
$R_K^{(0)}=0$ and defining
\[
    R_K^{(n+1)}
    :=
    \mathcal T_K R_K^{(n)},
\]
all the iterates are nonnegative, because $\bar\phi_K\ge0$, and
$R_K^{(n)}\to R_K$ in the weighted $L^1$-norm. Equivalently,
   $ R_K^{(n)}
    =
    \sum_{j=1}^{n}
    \bar\phi_K^{\ast j},$
so that $
    R_K
    =
    \sum_{j=1}^{\infty}
    \bar\phi_K^{\ast j}$
with convergence in the weighted $L^1$-norm
$\norm{\cdot}_{1,\alpha_K}$.  It remains to prove the uniform bound in $K$. Since
$\varphi_T\in L^1(0,T)$, dominated convergence gives
\[
    \norm{\varphi_T}_{1,\alpha}
    =
    \int_0^T
    e^{-\alpha t}\varphi_T(t)\,\dd t
    \longrightarrow 0
    \qquad\text{as }\alpha\to\infty.
\]
Choose $\alpha>0$ and $\rho\in(0,1)$ such that
\[
    \norm{\varphi_T}_{1,\alpha}<\rho<1.
\]
By \eqref{eq:phi-envelope-convergence},
\[
    \norm{
        \bar\phi_K-\varphi_T
    }_{1,\alpha}
    \le
    \norm{
        \bar\phi_K-\varphi_T
    }_{L^1(0,T)}
    \le
    \delta_K(T)
    \longrightarrow0.
\]
Hence, after enlarging $K_0$ if necessary,
\[
    \norm{\bar\phi_K}_{1,\alpha}
    \le
    \rho,
    \qquad K\ge K_0.
\] Taking the weighted $L^1$-norm in \eqref{eq:volterra} and using
\eqref{eq:weighted-convolution}, we obtain, for $K\ge K_0$,
\[
    \norm{R_K}_{1,\alpha}
    \le
    \norm{\bar\phi_K}_{1,\alpha}
    +
    \norm{\bar\phi_K}_{1,\alpha}
    \norm{R_K}_{1,\alpha}.
\]
Therefore,
\[
    \norm{R_K}_{1,\alpha}
    \le
    \frac{
        \norm{\bar\phi_K}_{1,\alpha}
    }{
        1-\norm{\bar\phi_K}_{1,\alpha}
    }
    \le
    \frac{\rho}{1-\rho}.
\]
Using the equivalence of the weighted and unweighted norms gives
\[
    \int_0^T R_K(t)\,\dd t
    \le
    e^{\alpha T}
    \norm{R_K}_{1,\alpha}
    \le
    e^{\alpha T}
    \frac{\rho}{1-\rho},
    \qquad K\ge K_0.
\]
For the finitely many indices $K<K_0$,
$\int_0^T R_K(t)\,\dd t<\infty$ by the existence argument above.
Combining these finitely many quantities with the preceding bound
yields
\[
    \sup_{K\ge1}
    \int_0^T R_K(t)\,\dd t
    <\infty.
\]
\end{proof}

\begin{remark}[Example: compact action set, separability and continuity of $Q_k(\cdot)$]
    Assume that  $m=1$ and $A$ is a compact set; the kernel $\Phi(\tau,a)$ is linearly separable $\Phi(\tau,a)=\hat{\Phi}(\tau)Q(a)$; the application $a\longmapsto Q(a)$ is continuous. Then we choose $\Phi_K(\tau,a)=Q(a)\hat{\Phi}_K(\tau),$ with $\hat{\Phi}_K(\tau):=\sum_{k=1}^K q_k e^{-\beta k\tau}$ for some coefficient $q_k$ such that $\hat{\Phi}_K\longrightarrow \hat{\Phi}$. Therefore, Assumption \ref{ass:kernel_approx_strong} is satisfied and so Lemma \ref{lemma:bounds}. This is in particular satisfied in the numerical section by choosing
    $\Phi(\tau,a)=Q_{Hill}(a) \hat{\Phi}(\tau),$
    with $A=[0,1]$ and
    $Q_{Hill}(a)=1-\frac{c_{eff} a}{a_{half}+a},\; c_{eff}>0,a_{half}\in A.$
\end{remark}

\subsection{Markovianization procedure}

We now turn to introduce the Markovian version of (Hawkes-SDE). We set a level of truncation $K>0$ and define the truncated Markovianized version of \eqref{eq:poisson-embedded-hawkes-sde} by

\begin{equation}
\label{eq:Markov-poisson-embedded-hawkes-sde}
\left\{
\begin{aligned}
X_t^{K,a}
={}&X_0
+\int_0^t b(s,X^{K,a}_{s-},a_s)\,ds
+\int_0^t \sigma(s,X^{K,a}_{s-},a_s)\,dW_s
\\
&\quad
+\sum_{i=1}^m
\int_{(0,t]}\int_0^\infty
\gamma_i(s,X^{K,a}_{s-},a_s)
\mathbf{1}_{\{\theta\leq\lambda_s^{K,a,i}\}}
\,\Pi^i(ds,d\theta),\; X_0\in \mathbb R^{d_x},
\\[0.25cm]
N_t^{K,a,i}
=&
\int_{(0,t]}\int_0^\infty
\mathbf{1}_{\{\theta\leq\lambda_s^{K,a,i}\}}
\,\Pi^i(ds,d\theta),
\qquad i=1,\ldots,m,
\\[0.25cm]
\lambda_t^{K,a,i}
={}&
\mu_i(t,X^{K,a}_{t-},a_t)
\\
&\quad
+\left(\sum_{j=1}^m
\int_{(0,t)}\int_0^\infty
(\Phi_K)_{ij}(t-s,a_t)
\mathbf{1}_{\{\theta\leq\lambda_s^{K,a,j}\}}
\,\Pi^j(ds,d\theta)\right)_+,
\qquad i=1,\ldots,m.
\end{aligned}
\right.
\end{equation} For each $k=1,\dots,K$, define the $\R^m$-valued exponentially-weighted
Hawkes memory factor $Z^{K,k}=(Z^{K,k,j})_{j=1}^m$, driven by the jumps of $N^K$, componentwise by
\[
Z_{t-}^{K,k,j}
:=
\int_{(0,t)} e^{-\beta k(t-s)}\,\dd N_s^{K,j}
=
\int_{(0,t)}\int_0^\infty
e^{-\beta k(t-s)}
\mathbf{1}_{\{\theta\leq\lambda_s^{K,j}\}}
\,\Pi^j(\dd s,\dd\theta),
\qquad j=1,\dots,m.
\]
Equivalently, $Z^{K,k}$ is the $\R^m$-valued c\`adl\`ag solution of the
linear jump-ODE 
\begin{equation}
\label{eq:memory-factor-ode}
\dd Z_t^{K,k}
=
-\beta k\,Z_t^{K,k}\,\dd t
+
\dd N_t^{K},
\qquad
Z_0^{K,k}=0,
\end{equation}
where $N^K=(N^{K,1},\dots,N^{K,m})^\top\in\R^m$; that is, the $j$-th
component decays at rate $\beta k$ and jumps by $1$ at each accepted
event of $N^{K,j}$. Therefore, we define the approximate memory of the Hawkes process by
\[
H_t^K(a_t):=\sum_{k=1}^K Q_k(a_t)\,Z_{t-}^{K,k}
=
\int_{(0,t)}\sum_{k=1}^K Q_k(a_t)e^{-\beta k(t-s)}\,\dd N_s^K
=
\int_{(0,t)}\Phi_K(t-s,a_t)\,\dd N_s^K,
\]
which recovers the Markovianized kernel
$\Phi_K(\tau,a)=\sum_{k=1}^K Q_k(a)e^{-\beta k\tau}$ of
Assumption~\ref{ass:kernel_approx_strong}.We define the c\`adl\`ag exponential memory variables
\[
    Z_t^{K,k}
    :=
    \int_0^t e^{-\beta k(t-s)}\,\dd N_s^K
    \in\R^m,
    \qquad k=1,\dots,K.
\]

Then, we define the signed approximate memory
\[
    H_t^K(a_t)
    :=
    \sum_{k=1}^K Q_k(a_t)Z_{t-}^{K,k}.
\] The approximating Markov Hawkes-SDE is thus given by
\[
(Hawkes-SDE)_K\begin{cases}
     \dd X_t^{K,a}
    =
    b(t,X_t^{K,a},a_t)\,\dd t
    +
    \sigma(t,X_t^{K,a},a_t)\,\dd W_t
    +
    \sum_{i=1}^m \gamma_i(t,X_{t-}^K,a_t)\,\dd N_t^{K,i},\\
    \lambda_t^K
    =
    \mu(t,X_{t-}^K,a_t)
    +
    \left(\sum_{k=1}^K Q_k(a_t)Z_{t-}^{K,k}\right)_+,\\
    \dd Z_t^{K,k}
    =
    -\beta k\,Z_t^{K,k}\,\dd t
    +
    \dd N_t^{K},
\end{cases}
\]
\noindent for a choice of $Q_k$ associated with the kernel $\Phi$
 as defined under Assumption \ref{ass:kernel_approx_strong}
 and where for any vector $u\in \mathbb R^m$, the vector $u_+$ denotes the componentwise positive part:
\[
    u_+=(\max\{u_1,0\},\dots,\max\{u_m,0\})^\top.
\]

\begin{remark}[Signed approximation and positivity]
Note that the positive part is set to preserve the usual density argument for finite linear combinations of $\{e^{-\beta kt}\}_{k\ge1}$ in $L^p(\R_+)$. Since signed approximants may create negative raw memory values, the approximating intensity applies a componentwise positive part only to the approximated memory term. Thus signed approximation and valid intensities are both retained.
\end{remark}

\begin{definition}[State Space]
    The finite-dimensional Markov state is given by 
   \[ S_t^K
    :=
    (t,X_t^{K,a},Z_t^{K,1},\dots,Z_t^{K,K}).\]
    We denote this space by $\Sspace_K$ where
\[
    \Sspace_K:=[0,T]\times\R^{d_x}\times(\R^m)^K.
\]

\end{definition} The process $S_t^K$ is c\`adl\`ag and Markov. Intensities and controls are evaluated using predictable information. Thus, when a Markov control is used in continuous time, the precise convention is
\[
    a_t=\pi(t,Y_{t-}^K),
    \qquad
    Y_t^K=(X_t^{K,a},Z_t^{K,1},\dots,Z_t^{K,K}).
\]
Inside Lebesgue time integrals we often write $\pi(t,Y_t^K)$ for readability, since $Y_t^K=Y_{t-}^K$ outside jump times and jump times have zero Lebesgue measure.

\begin{remark}[Model-free state construction]
The matrices $Q_k(a)$, the baseline $\mu$, and the intensity $\lambda$ are used only in the theoretical Markov approximation. They are not needed to construct the state $S_t^K$. The memory update induced by solving \eqref{eq:memory-factor-ode} only requires observed event times, component labels, and the chosen decay scale $\beta$. Hence the representation is compatible with model-free CT-DDPG.
\end{remark}

Similarly to the previous standing assumption (S4) we assume that the system $(Hawkes-SDE)_K$ admits a solution with enough integrability so that the following assumption is enforced along the study.

\begin{assumption}[Uniform well-posedness and localization moments
for the approximations]
\label{ass:approx_uniform_moments}
For every $K\geq1$ and every admissible control
$a\in\mathcal A$, the approximating Poisson-embedded system admits
a pathwise unique strong solution
\[
    (X^{K,a},N^{K,a},\lambda^{K,a})
\]
on $[0,T]$. Moreover, there exist $p_\star>1$ and $C_{p_\star,T}<\infty$, independent of
$K$ and $a$, such that for any $t\in [0,T]$
\[
\sup_{K\geq1}\sup_{a\in\mathcal A}
\E_{t,x,\mathbf z^K;a}\left[
    \sup_{t\leq u\leq T}\norm{X_u^{K,a}}^{p_\star}
    +
    \norm{N_T^{K,a}}_1^{p_\star}
    +
    \left(
        \int_0^T\norm{\lambda_t^{K,a}}_1\,\dd t
    \right)^{p_\star}
\right]
\leq
C_{p_\star,T},
\]
where $\E_{t,x,\mathbf z^K;a}[\cdot]:=\mathbb E[\cdot|X_t=x, (Z_t^{K,1},\dots,Z_t^{K,K})=\mathbf z^K]$ following the policy $a$.
\end{assumption}
We provide a sufficient condition for
Assumption~\ref{ass:approx_uniform_moments} in the proposition below, also satisfied in our numerical simulations (see Appendix \ref{app:setup}).
\begin{proposition}
\label{prop:sufficient_approx_uniform_moments}
Fix $p_\star>1$. In addition to
Assumption~\ref{ass:standing}(S1)--(S3) and \ref{ass:kernel_approx_strong}, suppose that the following conditions hold.

\begin{enumerate}

    \item[(i)] The baseline intensity is uniformly bounded:
    \[
        \sup_{(t,x,u)\in[0,T]\times\R^{d_x}\times A}
        \norm{\mu(t,x,u)}_1
        \leq \overline\mu
    \]
    for some $\overline\mu<\infty$.

    \item[(ii)] The jump amplitudes are uniformly bounded:
    \[
        \sup_{(t,x,u)\in[0,T]\times\R^{d_x}\times A}
        \sum_{i=1}^m\norm{\gamma_i(t,x,u)}
        \leq \Gamma
    \]
    for some $\Gamma<\infty$.

    \item[(iii)] The nonnegative kernel envelopes satisfy the uniform
    subcriticality condition
    \[
        \rho_T
        :=
        \sup_{K\geq1}
        \int_0^T\overline\phi_K(s)\,\dd s
        <1,
        \qquad
        \overline\phi_K(s)
        :=
        \sup_{u\in A}
        \norm{\Phi_K(s,u)}_{1\to1}.
    \]
\end{enumerate}

Then Assumption~\ref{ass:approx_uniform_moments} holds with exponent
$p_\star$.
\end{proposition}

\begin{proof}
Fix $K\geq1$ and $a\in\mathcal A$, and set
\[
    N_t^{K,a,\Sigma}
    :=
    \norm{N_t^{K,a}}_1
    =
    \sum_{i=1}^m N_t^{K,a,i},
\]
and
\[
    \ell_t^{K,a}
    :=
    \norm{\lambda_t^{K,a}}_1
    =
    \sum_{i=1}^m\lambda_t^{K,a,i}.
\] Since the componentwise positive-part map satisfies
\[
    \norm{z_+}_1\leq\norm{z}_1,
\]
we have the estimate
\begin{align}
\ell_t^{K,a}
&\leq
\norm{\mu(t,X_{t-}^{K,a},a_t)}_1
+
\left\|
    \int_{(0,t)}
    \Phi_K(t-s,a_t)\,\dd N_s^{K,a}
\right\|_1
\nonumber\\
&\leq
\overline\mu
+
\int_{(0,t)}
\overline\phi_K(t-s)\,
\dd N_s^{K,a,\Sigma}.
\label{eq:scalar_hawkes_envelope}
\end{align} Let $\overline N^K$ be the scalar linear Hawkes process with
constant baseline $\overline\mu$ and nonnegative kernel
$\overline\phi_K$:
\[
    \overline\lambda_t^K
    =
    \overline\mu
    +
    \int_{(0,t)}
    \overline\phi_K(t-s)\,\dd\overline N_s^K.
\]
With the common Poisson embedding and the monotonicity of the
right-hand side of \eqref{eq:scalar_hawkes_envelope}, we can deduce
\[
    N_t^{K,a,\Sigma}
    \leq
    \overline N_t^K,
    \qquad 0\leq t\leq T, \qquad\mbox{a.s.}
\] Extend $\overline\phi_K$ by zero outside $[0,T]$. In the cluster
representation of $\overline N^K$, each immigrant generates a Galton--Watson cluster whose number of direct offspring has Poisson mean \[
    \rho_K := \int_0^T\overline\phi_K(s)\,\dd s
    \leq\rho_T<1.
\] Let $M_T$ denote the number of immigrants of $ \overline N_t^K$ occurring in $[0,T]$, so we have $M_T\sim \mbox{Pois}(\overline{\mu}T)$. Let $S_{K,1}, S_{K,2},...$ be the independent total cluster sizes associated with these immigrants. Every point of $\overline N^K$ observed by time $T$ belongs to one of those clusters, so we have $$\overline{N}_T^K\le \sum_{j=1}^{M_T}S_{K,j}.$$ By Holder's inequality, we have \begin{equation*}
    \begin{split}
        \E\left[\left(\left.\sum_{j=1}^{M_T}S_{K,j}\right)^p\right|M_T= n\right]&\leq n^{p-1}\sum_{j=1}^n\E[S_{K,j}^p]\\
        &=n^p\E[S_{K}^p].
    \end{split}
\end{equation*} Therefore,\begin{equation*}
    \begin{split}
        \E\left[\left(\overline{N}_T^K\right)^p\right]&\le \E\left[\left(\sum_{j=1}^{M_T}S_{K,j}\right)^p\right]\\
        &\le \E[M_T^p]\E[S_{K}^p].
    \end{split}
\end{equation*} By Hawkes and Oakes~\cite[p.~496]{hawkes1974}, the cluster $S_K$ is
almost surely finite because $\rho_K<1$. In addition, its total
progeny has the Borel distribution; see
\cite{dwass1969total,tanner1961borel}. Since
$\rho_K\leq\rho_T<1$,  a Poisson($\rho_K$) offspring variable can be coupled as a thinning of a Poisson($\rho_T$) offspring variable. Applying this recursively to the Galton–Watson trees gives $S_K\le S_{\rho_T}$ a.s. under a suitable coupling. Then, because of the exponential tail
of the Borel$(\rho_T)$ distribution, we have \[
    \sup_{K\geq1}\E[S_K^{p_\star}]\le \mathbb E[S_{\rho_T}^{p^*}]<\infty.
\] Since the number of immigrants before $T$, $M_T,$ is bounded
by a Poisson random variable with mean $\overline\mu T$, it follows
that
\begin{equation}
\label{eq:uniform_count_moment_sufficient}
    \sup_{K\geq1}\sup_{a\in\mathcal A}
    \E\left[
        \bigl(N_T^{K,a,\Sigma}\bigr)^{p_\star}
    \right]
    <\infty.
\end{equation}
In particular, the approximating counting processes are
nonexplosive. Integrating \eqref{eq:scalar_hawkes_envelope} over time and applying
Fubini's theorem gives
\begin{align*}
\int_0^T\ell_t^{K,a}\,\dd t
&\leq
\overline\mu T
+
\int_{(0,T)}
\left(
    \int_s^T
    \overline\phi_K(t-s)\,\dd t
\right)
\dd N_s^{K,a,\Sigma}
\nonumber\\
&\leq
\overline\mu T
+
\rho_T N_T^{K,a,\Sigma}.
\end{align*}
Consequently, \eqref{eq:uniform_count_moment_sufficient} implies
\[
    \sup_{K\geq1}\sup_{a\in\mathcal A}
    \E\left[
        \left(
            \int_0^T
            \norm{\lambda_t^{K,a}}_1\,\dd t
        \right)^{p_\star}
    \right]
    <\infty.
\] It remains to estimate the physical state. The uniform boundedness
of the jump amplitudes gives
\[
    \sup_{0\leq u\leq t}
    \left\|
        \sum_{i=1}^m
        \int_{(0,u]}
        \gamma_i(s,X_{s-}^{K,a},a_s)
        \,\dd N_s^{K,a,i}
    \right\|
    \leq
    \Gamma N_t^{K,a,\Sigma}.
\]
Using the linear-growth assumptions on $b$ and $\sigma$, the
Burkholder--Davis--Gundy inequality, Young's inequality, we obtain
\[
\begin{aligned}
\E\left[
    \sup_{0\leq u\leq t}
    \norm{X_u^{K,a}}^{p_\star}
\right]
\leq
C_{p_\star,T}
\Bigg(
    1
    +
    \E\norm{X_0}^{p_\star}
    +
    \E\bigl[(
        N_T^{K,a,\Sigma}
    )^{p_\star}\bigr]
    +
    \int_0^t
    \E\left[
        \sup_{0\leq r\leq s}
        \norm{X_r^{K,a}}^{p_\star}
    \right]\dd s
\Bigg).
\end{aligned}
\]
Hence, Gronwall's inequality yields
\[
    \sup_{K\geq1}\sup_{a\in\mathcal A}
    \E\left[
        \sup_{0\leq t\leq T}
        \norm{X_t^{K,a}}^{p_\star}
    \right]
    <\infty.
\] Strong existence and pathwise uniqueness up to explosion follow from
the standard interlacing construction for jump-diffusions; see
\cite[Chapter~V.10, Theorem~57]{protter2005stochastic} and
\cite[Proposition~2.1]{banos2024hestonhawkes}. Between consecutive
accepted events, the physical state solves a globally Lipschitz
diffusion SDE, while at each accepted event the counting-process and
state updates are uniquely determined by the common Poisson
embedding and the pre-jump state. Furthermore, the pathwise
domination
\[
    N_t^{K,a,\Sigma}\leq\overline N_t^K
\]
and the almost-sure finiteness of the subcritical Hawkes count
$\overline N_T^K$ imply that only finitely many accepted events occur
on $[0,T]$. Hence explosion is impossible, and so we have a unique nonexplosive strong solution.
Together with the preceding uniform moment estimates, this proves
Assumption~\ref{ass:approx_uniform_moments}.
\end{proof}

\begin{remark}
    (ii) and (iii) are satisfied, for example if $\mu,\gamma$ are uniformly bounded in $x$ and $A$ is compact.
\end{remark}

\subsection{Objective functions and convergence of value functions}
We assume that the agent is minimizing a running cost $c:[0,T]\times \mathbb R^{d_x}\times A\longrightarrow \mathbb R$ along the time duration $[0,T]$ and a terminal cost $g:\mathbb R^{d_x} \longrightarrow\mathbb R$, satisfying the following assumption.

\begin{assumption}[Cost regularity]
\label{ass:reward}
$c$ is Borel measurable and continuous in time and action variable. In addition, there exist constants $C>0$ and $q\ge0$ such that,  uniformly in
$(t,a)$,
\[
    |c(t,x,a)|+|g(x)|
    \le
    C(1+\|x\|^{q+1}),
\]
and
\[
\begin{aligned}
    |c(t,x,a)-c(t,x',a)|
    &+
    |g(x)-g(x')|  \\
    &\le
    C\left(1+\|x\|^q+\|x'\|^q\right)\|x-x'\|.
\end{aligned}
\] For the objective and value convergence conclusions, the exponent
$p_\star$ in
Assumption~\ref{ass:approx_uniform_moments} is assumed to satisfy
\[
p_\star>q+1.
\]
\end{assumption} For an admissible control $a$, we define the objective function for the problem by
\[
    J(a)
    :=
    \E\left[\int_0^T c(t,X_t,a_t)\,\dd t+g(X_T)\right].
\] We aim at solving the following problem
\[
    V_0:=\inf_{a\in \mathcal A} J(a).
\] We define its Markovian approximation by
\[
    V_0^K:=\inf_{a\in \mathcal A} J_K(a),\quad J_K(a)
    :=
    \E\left[\int_0^T c(t,X_t^{K,a},a_t)\,\dd t+g(X_T^{K,a})\right].
\] Let $e_i\in\R^m$ be the $i$-th coordinate vector. For a smooth test function $f=f(t,x,\mathbf z^K)$, continuously differentiable with respect to the time component, twice continuously differentiable with respect to the diffusion variable $x$ and  one continuously differentiable with respect to each memory variable $z^i$, we define the extended generator 
\begin{align*}
\label{eq:generator}
    \mathcal L_K^a f(t,x,\mathbf z^K)
    &=
    \partial_t f(t,x,\mathbf z^K)
    +
    b(t,x,a)\cdot \nabla_x f(t,x,\mathbf z^K)
    +
    \frac12\Tr\left[\sigma\sigma^\top(t,x,a)\nabla_{xx}^2 f(t,x,\mathbf z^K)\right]
    \nonumber\\
    &\quad
    -
    \sum_{k=1}^K \beta k\, z^k\cdot \nabla_{z^k}f(t,x,\mathbf z^K)
    \nonumber\\
    &\quad
    +
    \sum_{i=1}^m
    \lambda_i^K(t,x,\mathbf z^K,a)
    \Big[
        f(t,x+\gamma_i(t,x,a),z^1+e_i,\dots,z^K+e_i)
        -f(t,x,\mathbf z^K)
    \Big],
\end{align*} where $\lambda_i^K(t,x,\mathbf z^K,a)  =\Big[\mu(t,x,a)
    +
    \left(\sum_{k=1}^K Q_k(a)z^k\right)_+\Big]_i.$

{
\begin{lemma}[Stopped physical-state stability]
\label{lem:stopped_physical_stability}
Suppose that Assumption~\ref{ass:standing}(S1) holds. Fix
$K\geq1$ and $a\in\mathcal A$. Let
\[
    (X^a,N^a,\lambda^a)
    \qquad\text{and}\qquad
    (X^{K,a},N^{K,a},\lambda^{K,a})
\]
be constructed with the same Brownian motion, the same Poisson random
measures, the same initial condition, and the same admissible control
process $a$. For $R\geq1$, define
\[
\begin{aligned}
\tau_R^{K,a}
:=
\inf\Bigl\{
t\in[0,T]:\;&
\norm{N_t^{K,a}}_1+\norm{N_t^a}_1>R
&\text{or }\norm{X_t^{K,a}}>R
\text{ or }\norm{X_t^a}>R
\Bigr\}
\wedge T.
\end{aligned}
\]
Set
\[
D_{X,R}^{K,a}(t)
:=
\E\left[
    \sup_{0\leq u\leq t}
    \norm{
        X_{u\wedge\tau_R^{K,a}}^{K,a}
        -
        X_{u\wedge\tau_R^{K,a}}^a
    }
\right],
\]
and
\[
D_{\lambda,R}^{K,a}(t)
:=
\E\int_0^{t\wedge\tau_R^{K,a}}
\norm{\lambda_s^{K,a}-\lambda_s^a}_1\,\dd s.
\] Then, for every $R\geq1$, there exists
$C_{T,R}<\infty$, independent of $K$ and
$a\in\mathcal A$, such that
\begin{equation}
\label{eq:stopped_x_stability}
D_{X,R}^{K,a}(t)
\leq
C_{T,R}\int_0^tD_{X,R}^{K,a}(s)\,\dd s
+
C_{T,R}D_{\lambda,R}^{K,a}(t),
\qquad
0\leq t\leq T.
\end{equation}
Consequently,
\begin{equation}
\label{eq:stopped_x_stability_closed}
D_{X,R}^{K,a}(t)
\leq
C_{T,R}D_{\lambda,R}^{K,a}(t),
\qquad
0\leq t\leq T.
\end{equation}
\end{lemma}

\begin{proof}
Fix $K\geq1$ and $a\in\mathcal A$. For readability, suppress the
superscript $a$, and write
\[
    \tau_R:=\tau_R^{K,a},
    \qquad
    \Delta X_t:=X_t^K-X_t.
\]
Define
\[
S_R(t)
:=
\sup_{0\leq u\leq t}
\norm{\Delta X_{u\wedge\tau_R}}.
\]
Thus,
\[
    D_{X,R}(t)=\E[S_R(t)].
\] For each $i=1,\ldots,m$, set
\[
I_s^{K,i}(\theta)
:=
\mathbf 1_{\{\theta\leq\lambda_s^{K,i}\}},
\qquad
I_s^i(\theta)
:=
\mathbf 1_{\{\theta\leq\lambda_s^i\}}.
\]
Define the counting process for the common jump from two processes by
\[
C_t^{K,i}
:=
\int_{(0,t]}\int_0^\infty
\mathbf 1_{\{
    \theta\leq
    \lambda_s^{K,i}\wedge\lambda_s^i
\}}
\,\Pi^i(\dd s,\dd\theta),
\]
and define the counting process for the discrepancy between two processes by
\[
\Xi_t^{K,i}
:=
\int_{(0,t]}\int_0^\infty
\left|
    I_s^{K,i}(\theta)-I_s^i(\theta)
\right|
\,\Pi^i(\dd s,\dd\theta).
\] Indeed, 
\begin{equation*}
        \int_{(0,t]}\int_{\mathbb R^+}\left|I_s^{K,i}(\theta) - I_s^i(\theta)\right|\Pi^i(\dd s,\dd\theta) =\int_{(0,t]}|dN_t^{K.i} - dN_t^i|,
\end{equation*} because for any fixed atom $(s,\theta)$, if $N^{K,i}_s$ and $N^K_s$ accept or reject it simultaneously, then $\left|I_s^{K,i}(\theta) - I_s^i(\theta)\right| = |dN_t^{K.i} - dN_t^i| = 0$; if they disagree on this point, then $\left|I_s^{K,i}(\theta) - I_s^i(\theta)\right| = |dN_t^{K.i} - dN_t^i| = 1$. Furthermore, \[\int_0^\infty
\left|
    I_s^{K,i}(\theta)-I_s^i(\theta)
\right|
\,\dd\theta
=
\abs{\lambda_s^{K,i}-\lambda_s^i},\] so we have 
\begin{equation}
    \label{eq:discrepancy_intensity}
    \mathbb E[\Xi_t^{K,i}] = \mathbb E\int_{(0,t]}\abs{\lambda_s^{K,i}-\lambda_s^i}ds
\end{equation}
Set
\[
C_t^K:=\sum_{i=1}^m C_t^{K,i},
\qquad
\Xi_t^K:=\sum_{i=1}^m \Xi_t^{K,i},
\] 

decomposing the total counts of two processes by common jumps and the discrepancies, we have
\begin{equation}
\label{eq:common_discrepancy_identity}
\norm{N_t^K}_1+\norm{N_t}_1
=
2C_t^K+\Xi_t^K.
\end{equation} Let
\[
    Q_t^K:=\norm{N_t^K}_1+\norm{N_t}_1.
\]
On a set of probability one, each component counting process has
jumps of size at most one. Consequently,
\[
    \Delta Q_t^K\leq 2m,
    \qquad 0\leq t\leq T.
\]
By the definition of $\tau_R$, $Q_{\tau_R-}^K\leq R$.
It follows that, for every $t\in[0,T]$,
\[
    Q_{t\wedge\tau_R}^K\leq R+2m.
\]
Using \eqref{eq:common_discrepancy_identity} and the nonnegativity of
$\Xi^K$, we obtain
\begin{equation}
\label{eq:stopped_common_count_bound}
C_{t\wedge\tau_R}^K
\leq
\frac{R+2m}{2}.
\end{equation} Write
\[
\gamma_s^{K,i}
:=
\gamma_i(s,X_{s-}^K,a_s),
\qquad
\gamma_s^i
:=
\gamma_i(s,X_{s-},a_s).
\]
The jump integrand admits the decomposition
\[
\begin{aligned}
\gamma_s^{K,i}I_s^{K,i}(\theta)
-
\gamma_s^iI_s^i(\theta)
={}&
\left(
    \gamma_s^{K,i}-\gamma_s^i
\right)
\mathbf 1_{\{
    \theta\leq
    \lambda_s^{K,i}\wedge\lambda_s^i
\}}
\\
&+
\gamma_s^{K,i}
\mathbf 1_{\{
    \lambda_s^i<\theta\leq\lambda_s^{K,i}
\}}
\\
&-
\gamma_s^i
\mathbf 1_{\{
    \lambda_s^{K,i}<\theta\leq\lambda_s^i
\}}.
\end{aligned}
\] Let $L$ be the Lipschitz and growth constant in
Assumption~\ref{ass:standing}(S1). By the definition of
$\tau_R$, for every $s\leq\tau_R$,$\norm{X_{s-}^K}\vee\norm{X_{s-}}\leq R$.
Therefore,
\[
\norm{\gamma_s^{K,i}-\gamma_s^i}
\leq
L S_R(s-),
\]
and
\[
\norm{\gamma_s^{K,i}}
\vee
\norm{\gamma_s^i}
\leq
\Gamma_R,
\qquad
\Gamma_R:=L(1+R).
\]
Consequently,
\[
\norm{
\gamma_s^{K,i}I_s^{K,i}(\theta)
-
\gamma_s^iI_s^i(\theta)
}\leq
L_\gamma S_R(s-)
\mathbf 1_{\{
    \theta\leq
    \lambda_s^{K,i}\wedge\lambda_s^i
\}}
+
\Gamma_R
\left|
    I_s^{K,i}(\theta)-I_s^i(\theta)
\right|.
\] Subtracting the two stopped state equations, and using the Lipschitz
continuity of $b$, gives
\begin{equation}
\label{eq:stopped_state_pathwise}
S_R(t)
\leq{}
L\int_0^tS_R(s)\,\dd s
+
M_R(t)
+
\Gamma_R\Xi_{t\wedge\tau_R}^K+
L
\int_{(0,t\wedge\tau_R]}
S_R(s-)\,\dd C_s^K,
\end{equation}
where
\[
M_R(t)
:=
\sup_{0\leq u\leq t}
\left\|
\int_0^{u\wedge\tau_R}
\left[
    \sigma(s,X_{s-}^K,a_s)
    -
    \sigma(s,X_{s-},a_s)
\right]\dd W_s
\right\|.
\] Set
\[
H_R(t)
:=
L\int_0^tS_R(s)\,\dd s
+
M_R(t)
+
\Gamma_R\Xi_{t\wedge\tau_R}^K.
\]
The process $H_R$ is nonnegative and nondecreasing. Iterating
\eqref{eq:stopped_state_pathwise} over the jump times of $C^K$
gives
\[
S_R(t)
\leq
H_R(t)
\prod_{0<s\leq t\wedge\tau_R}
\left(
    1+L\Delta C_s^K
\right).
\]
Since $1+x\leq e^x$ for $x\geq0$,
\[
\prod_{0<s\leq t\wedge\tau_R}
\left(
    1+L\Delta C_s^K
\right)
\leq
\exp\left(
    L C_{t\wedge\tau_R}^K
\right).
\]
Using \eqref{eq:stopped_common_count_bound}, we conclude that
\begin{equation}
\label{eq:finite_jump_amplification}
S_R(t)
\leq
\kappa_R
\left[
    L\int_0^tS_R(s)\,\dd s
    +
    M_R(t)
    +
    \Gamma_R\Xi_{t\wedge\tau_R}^K
\right],
\end{equation}
where
\[
    \kappa_R
    :=
    \exp\left(
        \frac{L(R+2m)}{2}
    \right).
\] By the Burkholder--Davis--Gundy inequality and the Lipschitz
continuity of $\sigma$,
\[
\begin{aligned}
\E[M_R(t)]
&\leq
C_{\mathrm{BDG}}L
\E\left[
    \left(
        \int_0^{t\wedge\tau_R}
        \norm{\Delta X_{s-}}^2\,\dd s
    \right)^{1/2}
\right].
\end{aligned}
\]
Moreover,
\[
\int_0^{t\wedge\tau_R}
\norm{\Delta X_{s-}}^2\,\dd s
\leq
S_R(t)\int_0^tS_R(s)\,\dd s.
\]
Hence, by the Cauchy--Schwarz inequality,
\[
\begin{aligned}
\E[M_R(t)]
&\leq
C_{\mathrm{BDG}}L
\E\left[
    \left(
        S_R(t)\int_0^tS_R(s)\,\dd s
    \right)^{1/2}
\right]
\\
&\leq
C_{\mathrm{BDG}}L
\left(
    D_{X,R}(t)
    \int_0^tD_{X,R}(s)\,\dd s
\right)^{1/2}.
\end{aligned}
\]
Let $A = \kappa_RC_{\mathrm{BDG}}L(\int_0^tD_{X,R}(s)\,\dd s)^{1/2}$, $B = (D_{X,R}(t))^{1/2}$, then
we have $AB\leq \frac{1}{2}A^2 + \frac{1}{2}B^2$, which implies
\begin{equation}
\label{eq:stopped_BDG_bound}
\kappa_R\E[M_R(t)]
\leq
\frac12D_{X,R}(t)
+
C_{T,R}\int_0^tD_{X,R}(s)\,\dd s,
\end{equation} where $C_{T,R}$ is constant depending on $T,R$. Note that, by \eqref{eq:discrepancy_intensity}
\[
\Xi_t^K
-
\int_0^t
\norm{\lambda_s^K-\lambda_s}_1\,\dd s
\]
is a local martingale. Therefore, from \eqref{eq:discrepancy_intensity} and Fubini's theorem,
\begin{align}
\label{eq:discrepancy_compensator_refined}
\E\left[
    \Xi_{t\wedge\tau_R}^K
\right]
&=
\sum_{i=1}^m
\E\int_{(0,T]}\int_0^\infty
\mathbf 1_{\{s\leq t\wedge\tau_R\}}
\left|
    \mathbf 1_{\{\theta\leq\lambda_s^{K,i}\}}
    -
    \mathbf 1_{\{\theta\leq\lambda_s^i\}}
\right|
\Pi^i(\dd s,\dd\theta)
\nonumber\\
&=
\sum_{i=1}^m
\E\int_0^T\int_0^\infty
\mathbf 1_{\{s\leq t\wedge\tau_R\}}
\left|
    \mathbf 1_{\{\theta\leq\lambda_s^{K,i}\}}
    -
    \mathbf 1_{\{\theta\leq\lambda_s^i\}}
\right|
\dd\theta\,\dd s
\nonumber\\
&=
\E\int_0^{t\wedge\tau_R}
\sum_{i=1}^m
\abs{\lambda_s^{K,i}-\lambda_s^i}
\,\dd s
\nonumber\\
&=
\E\int_0^{t\wedge\tau_R}
\norm{\lambda_s^K-\lambda_s}_1\,\dd s
\nonumber\\
&=
D_{\lambda,R}(t).
\end{align}
Notice that
$\Xi_{t\wedge\tau_R}^K\leq Q_{t\wedge\tau_R}^K\leq R+2m$,
so the stopped discrepancy count is integrable. Taking expectations in
\eqref{eq:finite_jump_amplification}, using
\eqref{eq:stopped_BDG_bound} and
\eqref{eq:discrepancy_compensator_refined}, gives
\[
\begin{aligned}
D_{X,R}(t)
\leq{}&
C_{T,R}\int_0^tD_{X,R}(s)\,\dd s
+
\frac12D_{X,R}(t)
+
C_{T,R}D_{\lambda,R}(t).
\end{aligned}
\]
Absorbing the term $\frac12D_{X,R}(t)$ into the left-hand side
yields
\[
D_{X,R}(t)
\leq
C_{T,R}\int_0^tD_{X,R}(s)\,\dd s
+
C_{T,R}D_{\lambda,R}(t).
\]
This proves \eqref{eq:stopped_x_stability}. Finally, $D_{\lambda,R}$ is nonnegative and nondecreasing.
The inhomogeneous Gronwall inequality therefore gives
\[
D_{X,R}(t)
\leq
C_{T,R}D_{\lambda,R}(t),
\qquad
0\leq t\leq T,
\]
after enlarging $C_{T,R}$ if necessary. This proves
\eqref{eq:stopped_x_stability_closed}.
\end{proof}

}

We now give a fundamental lemma related to Volterra resolvent. 

\begin{lemma}[Volterra-resolvent comparison]
\label{lem:volterra_resolvent_comparison}
Let $k\in L^1([0,T];\mathbb R_+)$, and let $R$ be the solution of the Volterra equation
\[
    R(t)=k(t)+\int_0^t k(t-s)R(s)\,\dd s, \ t\in [0,T].
\]
Suppose $f,F$ are nonnegative integrable functions on $[0,T]$ and
\[
    f(t)
    \le
    F(t)+\int_0^t k(t-s)f(s)\,\dd s,
    \qquad 0\le t\le T.
\]

Then
\[
    f(t)
    \le
    F(t)+\int_0^t R(t-s)F(s)\,\dd s.
\]
Consequently,
\[
    \int_0^t f(u)\,\dd u
    \le
    \left(
        1+\int_0^T R(u)\,\dd u
    \right)
    \int_0^t F(s)\,\dd s,
    \qquad 0\le t\le T.
\]
\end{lemma}

\begin{proof}
Define
\[
    G(t)
    :=
    F(t)+\int_0^t R(t-s)F(s)\,\dd s.
\]
Using the resolvent identity $R=k+k*R$, we have
\[
    G=F+R*F=F+k*F+k*R*F=F+k*G.
\]
Since $f\le F+k*f$ and $G=F+k*G$, the Volterra comparison principle (see \cite{Beesack1969}) gives $f\le G$. Hence
\[
    f(t)
    \le
    F(t)+\int_0^t R(t-s)F(s)\,\dd s.
\]

Integrating over $[0,t]$ gives
\[
\begin{aligned}
    \int_0^t f(u)\,\dd u
    &\le
    \int_0^t F(u)\,\dd u
    +
    \int_0^t\int_0^u R(u-s)F(s)\,\dd s\,\dd u \\
    &=
    \int_0^t F(s)
    \left[
        1+\int_s^t R(u-s)\,\dd u
    \right]\dd s \\
    &\le
    \left(
        1+\int_0^T R(r)\,\dd r
    \right)
    \int_0^t F(s)\,\dd s.
\end{aligned}
\]
\end{proof}

{\begin{theorem}[Convergence of the Markov approximation]
\label{thm:convergence}
Suppose Assumptions~\ref{ass:standing},
\ref{ass:kernel_approx_strong}, and
\ref{ass:approx_uniform_moments} hold. For every
$K\geq1$ and $a\in\mathcal A$, let
\[
(X^a,N^a,\lambda^a)
\qquad\text{and}\qquad
(X^{K,a},N^{K,a},\lambda^{K,a})
\]
be the respective solutions of
\eqref{eq:poisson-embedded-hawkes-sde} and
\eqref{eq:Markov-poisson-embedded-hawkes-sde}, constructed with the
same Brownian motion, the same Poisson random measures, and the same
admissible control process $a=(a_t)_{0\leq t\leq T}$. Then
\begin{equation}
\label{eq:process_convergence}
\lim_{K\to\infty}
\sup_{a\in\mathcal A}
\left\{
\E\left[
    \sup_{0\leq t\leq T}
    \norm{X_t^{K,a}-X_t^a}
\right]
+
\E\int_0^T
\norm{\lambda_t^{K,a}-\lambda_t^a}_1\,\dd t
\right\}
=0.
\end{equation}
Moreover,
\begin{equation}
\label{eq:counting_process_convergence}
\lim_{K\to\infty}
\sup_{a\in\mathcal A}
\E\left[
    \sum_{i=1}^m
    \sup_{0\leq t\leq T}
    \abs{N_t^{K,a,i}-N_t^{a,i}}
\right]
=0.
\end{equation} If, in addition, Assumption~\ref{ass:reward} holds with its growth
exponent $q$ satisfying
\[
p_\star>q+1,
\]
where $p_\star$ is the exponent in
Assumption~\ref{ass:approx_uniform_moments}, then
\[
\sup_{a\in\mathcal A}\abs{J_K(a)-J(a)}\longrightarrow0.
\]
Consequently, for
\[
V_0:=\inf_{a\in\mathcal A}J(a),
\qquad
V_0^K:=\inf_{a\in\mathcal A}J_K(a),
\]
we have
\[
\abs{V_0^K-V_0}\longrightarrow0.
\]
\end{theorem}

\begin{proof}
Fix $K\geq1$ and $a\in\mathcal A$. Throughout the
process-convergence part of the proof, suppress the superscript
$a$. Let
$\tau_R:=\tau_R^{K,a}$
be the stopping time introduced in
Lemma~\ref{lem:stopped_physical_stability}. Thus,
\[
\begin{aligned}
\tau_R
=
\inf\Bigl\{
t\in[0,T]:\;&
\norm{N_t^K}_1+\norm{N_t}_1>R
&\text{or }\norm{X_t^K}>R
\text{ or }\norm{X_t}>R
\Bigr\}
\wedge T.
\end{aligned}
\] Define
\[
D_{X,R}(t)
:=
\E\left[
    \sup_{0\leq u\leq t}
    \norm{
        X_{u\wedge\tau_R}^K-X_{u\wedge\tau_R}
    }
\right],\quad D_{\lambda,R}(t)
:=
\E\int_0^{t\wedge\tau_R}
\norm{\lambda_s^K-\lambda_s}_1\,\dd s.
\]
Also set
\[
d_{\lambda,R}(t)
:=
\E\left[
    \mathbf 1_{\{t\leq\tau_R\}}
    \norm{\lambda_t^K-\lambda_t}_1
\right].
\]
Since the singleton $\{\tau_R\}$ has zero Lebesgue measure, we have
\[
D_{\lambda,R}(t)
=
\int_0^t d_{\lambda,R}(s)\,\dd s.
\] For each component $i=1,\ldots,m$, recall the discrepancy counting
process is defined as
\[
\Xi_t^{K,i}
:=
\int_{(0,t]}\int_0^\infty
\left|
    \mathbf 1_{\{\theta\leq\lambda_s^{K,i}\}}
    -
    \mathbf 1_{\{\theta\leq\lambda_s^i\}}
\right|
\Pi^i(\dd s,\dd\theta)\quad\mbox{and}\quad
    \Xi_t^K:=\sum_{i=1}^m\Xi_t^{K,i}.
\] For every $i$,
\[
\int_0^\infty
\left|
    \mathbf 1_{\{\theta\leq\lambda_s^{K,i}\}}
    -
    \mathbf 1_{\{\theta\leq\lambda_s^i\}}
\right|
\,\dd\theta
=
\abs{\lambda_s^{K,i}-\lambda_s^i}.
\]
Moreover,
\[
(s,\theta)
\longmapsto
\mathbf 1_{\{s\leq t\wedge\tau_R\}}
\left|
    \mathbf 1_{\{\theta\leq\lambda_s^{K,i}\}}
    -
    \mathbf 1_{\{\theta\leq\lambda_s^i\}}
\right|
\]
is nonnegative and predictable. 

\paragraph{Step 1: global control of the kernel-approximation error.}

Define
\[
A_K(t)
:=
\E\left\|
    \int_{(0,t)}
    \left[
        \Phi_K(t-s,a_t)-\Phi(t-s,a_t)
    \right]\dd N_s
\right\|_1.
\]
By the definition of
$\bar\varepsilon_K$,
\[
A_K(t)
\leq
\E\int_{(0,t)}
\bar\varepsilon_K(t-s)\,
\dd\norm{N_s}_1.
\]
Consequently, Fubini's theorem and the compensator of $N$ give
\begin{align*}
\int_0^T A_K(t)\,\dd t
&\leq
\E\int_0^T\int_{(0,t)}
\bar\varepsilon_K(t-s)\,
\dd\norm{N_s}_1\,\dd t
\nonumber\\
&=
\E\int_{(0,T)}
\left(
    \int_s^T
    \bar\varepsilon_K(t-s)\,\dd t
\right)
\dd\norm{N_s}_1
\nonumber\\
&=
\E\int_0^T
\left(
    \int_s^T
    \bar\varepsilon_K(t-s)\,\dd t
\right)
\norm{\lambda_s}_1\,\dd s
\nonumber\\
&\leq
\delta_K(T)
\E\int_0^T\norm{\lambda_s}_1\,\dd s
\nonumber\\
&\leq
C_T\delta_K(T).
\end{align*}
The last estimate is uniform over $a\in\mathcal A$ by
Assumption~\ref{ass:standing}(S4).

\paragraph{Step 2: intensity stability before $\tau_R$.}

Define the predictable raw memories
\[
\widehat H_t^K
:=
\int_{(0,t)}
\Phi_K(t-s,a_t)\,\dd N_s^K,
\qquad
H_t
:=
\int_{(0,t)}
\Phi(t-s,a_t)\,\dd N_s.
\]
The two intensities satisfy
\[
\lambda_t^K
=
\mu(t,X_{t-}^K,a_t)+(\widehat H_t^K)_+,
\]
and, since $\Phi$ is entrywise nonnegative,
\[
\lambda_t
=
\mu(t,X_{t-},a_t)+H_t
=
\mu(t,X_{t-},a_t)+(H_t)_+.
\]
The componentwise positive-part map is $1$-Lipschitz in the
$\ell^1$-norm. Hence, by
Assumption~\ref{ass:standing}(S2),
\begin{equation}
\label{eq:lambda_raw_bound_refined}
\norm{\lambda_t^K-\lambda_t}_1
\leq
L\norm{X_{t-}^K-X_{t-}}
+
\norm{\widehat H_t^K-H_t}_1.
\end{equation} Decompose
\begin{align}
\label{eq:memory_decomposition_refined}
\widehat H_t^K-H_t
=
\int_{(0,t)}
\Phi_K(t-s,a_t)\,
\dd(N_s^K-N_s)
+
\int_{(0,t)}
\left[
    \Phi_K(t-s,a_t)-\Phi(t-s,a_t)
\right]\dd N_s.
\end{align} For the first term, define
\[
B_{K,R}(t)
:=
\E\left[
\mathbf 1_{\{t\leq\tau_R\}}
\left\|
    \int_{(0,t)}
    \Phi_K(t-s,a_t)\,
    \dd(N_s^K-N_s)
\right\|_1
\right].
\]
Pathwise, the total variation of
$\dd(N^K-N)$ is bounded by the discrepancy count. Therefore,
\begin{align}
B_{K,R}(t)
&\leq
\E\left[
\mathbf 1_{\{t\leq\tau_R\}}
\sum_{i=1}^m
\int_{(0,t)}
\bar\phi_K(t-s)\,
\dd\Xi_s^{K,i}
\right]
\nonumber\\
&\leq
\E\sum_{i=1}^m
\int_{(0,t)}
\mathbf 1_{\{s\leq\tau_R\}}
\bar\phi_K(t-s)\,
\dd\Xi_s^{K,i}
\nonumber\\
&=
\int_0^t
\bar\phi_K(t-s)
\E\left[
    \mathbf 1_{\{s\leq\tau_R\}}
    \norm{\lambda_s^K-\lambda_s}_1
\right]\dd s
\nonumber\\
&=
\int_0^t
\bar\phi_K(t-s)
d_{\lambda,R}(s)\,\dd s.
\label{eq:BK_refined_bound}
\end{align}
In the second line, we used $\mathbf 1_{\{t\leq\tau_R\}}
\leq
\mathbf 1_{\{s\leq\tau_R\}}$ for any $s<t,$
and the third line follows from the compensation formula. Multiplying \eqref{eq:lambda_raw_bound_refined} by
$\mathbf 1_{\{t\leq\tau_R\}}$, taking expectations, and using
\eqref{eq:memory_decomposition_refined},
\eqref{eq:BK_refined_bound}, and the global quantity $A_K(t)$,
gives
\[
d_{\lambda,R}(t)
\leq
L D_{X,R}(t)
+
A_K(t)
+
\int_0^t
\bar\phi_K(t-s)
d_{\lambda,R}(s)\,\dd s.
\] Apply Lemma~\ref{lem:volterra_resolvent_comparison} with $ f=d_{\lambda,R},
    \;  F=L D_{X,R}+A_K,
    \;
    k=\bar\phi_K.$
Using the uniform resolvent bound gives
\begin{align}
D_{\lambda,R}(t)
&\leq
C_T
\left[
    \int_0^tD_{X,R}(s)\,\dd s
    +
    \int_0^tA_K(s)\,\dd s
\right]
\nonumber\\
&\leq
C_T\int_0^tD_{X,R}(s)\,\dd s
+
C_T\delta_K(T),
\qquad 0\leq t\leq T.
\label{eq:stopped_lambda_stability_refined}
\end{align}
Here $C_T$ is independent of $K$, $a$, and $R$. By Lemma~\ref{lem:stopped_physical_stability},
\[
D_{X,R}(t)
\leq
C_{T,R}D_{\lambda,R}(t).
\]
Substituting this estimate into
\eqref{eq:stopped_lambda_stability_refined} gives
\[
D_{\lambda,R}(t)
\leq
C_{T,R}
\int_0^tD_{\lambda,R}(s)\,\dd s
+
C_T\delta_K(T).
\]
Gronwall's inequality implies
\[
D_{\lambda,R}(T)
\leq
C_{T,R}\delta_K(T).
\]
Applying Lemma~\ref{lem:stopped_physical_stability} once more gives
\[
D_{X,R}(T)
\leq
C_{T,R}\delta_K(T).
\]
Therefore, for every fixed $R\geq1$,
\begin{equation}
\label{eq:fixed_R_refined_convergence}
\lim_{K\to\infty}
\sup_{a\in\mathcal A}
\left[
    D_{X,R}^{K,a}(T)
    +
    D_{\lambda,R}^{K,a}(T)
\right]
=0.
\end{equation}

\paragraph{Step 3: removal of the state/total-count localization.}

Let $p:=p_\star.$ By Markov's inequality,
\begin{align}
\Prob(\tau_R^{K,a}<T)
&\leq
\frac{1}{R^p}
\E\left[
    \sup_{0\leq t\leq T}
    \norm{X_t^{K,a}}^p
\right]
+
\frac{1}{R^p}
\E\left[
    \sup_{0\leq t\leq T}
    \norm{X_t^a}^p
\right]
\nonumber\\
&\quad+
\frac{1}{R^p}
\E\left[
    \left(
        \norm{N_T^{K,a}}_1+\norm{N_T^a}_1
    \right)^p
\right].
\label{eq:localization_probability_intermediate}
\end{align}
Since
\[
(x+y)^p
\leq
2^{p-1}(x^p+y^p),
\qquad x,y\geq0,
\]
Assumptions~\ref{ass:standing}(S4) and
\ref{ass:approx_uniform_moments} imply
\begin{equation}
\label{eq:refined_localization_probability}
\sup_{K\geq1}\sup_{a\in\mathcal A}
\Prob(\tau_R^{K,a}<T)
\leq
\frac{C_{p,T}}{R^p}.
\end{equation} For the physical states,
\begin{align}
&\E\left[
    \sup_{0\leq t\leq T}
    \norm{X_t^{K,a}-X_t^a}
\right]
\leq
D_{X,R}^{K,a}(T)
+
\E\left[
\left(
    \sup_{0\leq t\leq T}\norm{X_t^{K,a}}
    +
    \sup_{0\leq t\leq T}\norm{X_t^a}
\right)
\mathbf 1_{\{\tau_R^{K,a}<T\}}
\right].
\label{eq:state_remove_localization}
\end{align}
By Hölder's inequality,
\begin{align*}
&
\E\left[
\left(
    \sup_{0\leq t\leq T}\norm{X_t^{K,a}}
    +
    \sup_{0\leq t\leq T}\norm{X_t^a}
\right)
\mathbf 1_{\{\tau_R^{K,a}<T\}}
\right]
\\
&\quad\leq
\left\|
    \sup_{0\leq t\leq T}\norm{X_t^{K,a}}
    +
    \sup_{0\leq t\leq T}\norm{X_t^a}
\right\|_{L^p}
\Prob(\tau_R^{K,a}<T)^{1-\frac1p}.
\end{align*}
Using the uniform moment bounds and
\eqref{eq:refined_localization_probability}, we obtain
\begin{equation}
\label{eq:state_refined_tail}
\sup_{K\geq1}\sup_{a\in\mathcal A}
\E\left[
\left(
    \sup_{0\leq t\leq T}\norm{X_t^{K,a}}
    +
    \sup_{0\leq t\leq T}\norm{X_t^a}
\right)
\mathbf 1_{\{\tau_R^{K,a}<T\}}
\right]
\leq
\frac{C_{p,T}}{R^{p-1}}.
\end{equation} Similarly,
\begin{align}
&\E\int_0^T
\norm{\lambda_t^{K,a}-\lambda_t^a}_1\,\dd t \leq
D_{\lambda,R}^{K,a}(T)
+
\E\left[
\left(
    \int_0^T
    \bigl(
        \norm{\lambda_t^{K,a}}_1
        +
        \norm{\lambda_t^a}_1
    \bigr)\,\dd t
\right)
\mathbf 1_{\{\tau_R^{K,a}<T\}}
\right].
\label{eq:intensity_remove_localization}
\end{align}
Hölder's inequality, the moment assumptions, and
\eqref{eq:refined_localization_probability} give
\begin{equation}
\label{eq:intensity_refined_tail}
\sup_{K\geq1}\sup_{a\in\mathcal A}
\E\left[
\left(
    \int_0^T
    \bigl(
        \norm{\lambda_t^{K,a}}_1
        +
        \norm{\lambda_t^a}_1
    \bigr)\,\dd t
\right)
\mathbf 1_{\{\tau_R^{K,a}<T\}}
\right]
\leq
\frac{C_{p,T}}{R^{p-1}}.
\end{equation} Combining
\eqref{eq:fixed_R_refined_convergence},
\eqref{eq:state_remove_localization},
\eqref{eq:state_refined_tail},
\eqref{eq:intensity_remove_localization}, and
\eqref{eq:intensity_refined_tail}, we obtain
\[
\begin{aligned}
\limsup_{K\to\infty}
\sup_{a\in\mathcal A}
\Bigg\{
&
\E\left[
    \sup_{0\leq t\leq T}
    \norm{X_t^{K,a}-X_t^a}
\right]+
\E\int_0^T
\norm{\lambda_t^{K,a}-\lambda_t^a}_1\,\dd t
\Bigg\}
\leq
\frac{C_{p,T}}{R^{p-1}}.
\end{aligned}
\]
Letting $R\to\infty$ proves
\eqref{eq:process_convergence}.

\paragraph{Step 4: convergence of the counting processes.}

For each component $i$, the common Poisson embedding gives the
pathwise total-variation bound
\[
\sup_{0\leq t\leq T}
\abs{N_t^{K,a,i}-N_t^{a,i}}
\leq
\Xi_T^{K,a,i}.
\]
By \eqref{eq:discrepancy_intensity},
\[
\begin{aligned}
\E\left[
    \sum_{i=1}^m
    \sup_{0\leq t\leq T}
    \abs{N_t^{K,a,i}-N_t^{a,i}}
\right]
\leq
\E\left[
    \sum_{i=1}^m\Xi_T^{K,a,i}
\right]=
\E\int_0^T
\norm{\lambda_t^{K,a}-\lambda_t^a}_1\,\dd t.
\end{aligned}
\]
Taking the supremum over $a\in\mathcal A$ and using
\eqref{eq:process_convergence} proves
\eqref{eq:counting_process_convergence}.

\paragraph{Step 5: objective and value convergence.}

For every $a\in\mathcal A$, define
\[
    \Delta_K^a
    :=
    \sup_{0\le t\le T}
    \norm{X_t^{K,a}-X_t^a},
    \qquad
    U_K^a
    :=
    \max\left\{
        \sup_{0\le t\le T}\norm{X_t^{K,a}},
        \sup_{0\le t\le T}\norm{X_t^a}
    \right\}.
\]
By the weighted Lipschitz condition on $c$ and $g$,
\begin{align}
    \abs{J_K(a)-J(a)}
    &\le
    C
    \E\left[
        \int_0^T
        \left(
            1+\norm{X_t^{K,a}}^q+\norm{X_t^a}^q
        \right)
        \norm{X_t^{K,a}-X_t^a}
        \,\dd t
    \right]
    \notag\\
    &\quad
    +
    C
    \E\left[
        \left(
            1+\norm{X_T^{K,a}}^q+\norm{X_T^a}^q
        \right)
        \norm{X_T^{K,a}-X_T^a}
    \right]
    \notag\\
    &\le
    C_T
    \E\left[
        \left(
            1+(U_K^a)^q
        \right)
        \Delta_K^a
    \right].
    \label{eq:cost_error_interpolation}
\end{align} If $q=0$, then \eqref{eq:process_convergence} directly gives
\[
    \sup_{a\in\mathcal A}
    \abs{J_K(a)-J(a)}
    \le
    C_T
    \sup_{a\in\mathcal A}
    \E[\Delta_K^a]
    \longrightarrow0.
\] Suppose now that $q>0$, and set
$    s:=\frac{p_\star}{q},$ and $
    r:=\frac{p_\star}{p_\star-q}.$
Then $r,s>1$ and $
    \frac1r+\frac1s=1.$
Moreover, since $
    p_\star>q+1$ then
    $p_\star-q>1$ and so $
    r<p_\star.$ First note that Assumptions~\ref{ass:standing}(S4) and
\ref{ass:approx_uniform_moments} imply
\begin{equation}
\label{eq:uniform-U-Delta-moments}
    \sup_{K\ge1}\sup_{a\in\mathcal A}
    \E\left[
        (U_K^a)^{p_\star}
        +
        (\Delta_K^a)^{p_\star}
    \right]
    <\infty.
\end{equation}
Indeed,
\[
    (U_K^a)^{p_\star}
    \le
    \sup_{0\le t\le T}\norm{X_t^{K,a}}^{p_\star}
    +
    \sup_{0\le t\le T}\norm{X_t^a}^{p_\star},
\]
and
\[
    (\Delta_K^a)^{p_\star}
    \le
    2^{p_\star-1}
    \left(
        \sup_{0\le t\le T}\norm{X_t^{K,a}}^{p_\star}
        +
        \sup_{0\le t\le T}\norm{X_t^a}^{p_\star}
    \right).
\]  Define
\[
    \vartheta
    :=
    \frac{p_\star-q-1}{p_\star-1}.
\]
Since $q>0$ and $p_\star>q+1$, we have
\[
    0<\vartheta<1,
    \qquad
    1-\vartheta
    =
    \frac{q}{p_\star-1}.
\]
Furthermore,
\begin{align*}
    \vartheta+\frac{1-\vartheta}{p_\star}
    &=
    \frac{p_\star-q-1}{p_\star-1}
    +
    \frac{q}{p_\star(p_\star-1)}\\
    &=
    \frac{p_\star-q}{p_\star}
    =
    \frac1r.
\end{align*} Define the conjugate Hölder exponents
\[
    \kappa_1
    :=
    \frac{1}{r\vartheta},
    \qquad
    \kappa_2
    :=
    \frac{p_\star}{r(1-\vartheta)}.
\]
Then we have
\[
    \frac1{\kappa_1}+\frac1{\kappa_2}
    =
    r\vartheta
    +
    \frac{r(1-\vartheta)}{p_\star}
    =
    1,
\]
and therefore $\kappa_1,\kappa_2>1$. Applying Hölder's
inequality yields
\begin{align*}
    \E\left[(\Delta_K^a)^r\right]
    &=
    \E\left[
        (\Delta_K^a)^{r\vartheta}
        (\Delta_K^a)^{r(1-\vartheta)}
    \right]\\
    &\le
    \E\left[
        (\Delta_K^a)^{r\vartheta\kappa_1}
    \right]^{1/\kappa_1}
    \E\left[
        (\Delta_K^a)^{r(1-\vartheta)\kappa_2}
    \right]^{1/\kappa_2}\\
    &=
    \E[\Delta_K^a]^{r\vartheta}
    \E\left[
        (\Delta_K^a)^{p_\star}
    \right]^{r(1-\vartheta)/p_\star}.
\end{align*}
Taking the $r$-th root gives the inequality
\begin{equation}
\label{eq:explicit-interpolation}
    \E\left[(\Delta_K^a)^r\right]^{1/r}
    \le
    \E[\Delta_K^a]^\vartheta
    \E\left[
        (\Delta_K^a)^{p_\star}
    \right]^{(1-\vartheta)/p_\star}.
\end{equation}
Equivalently,
\[
    \norm{\Delta_K^a}_{L^r}
    \le
    \norm{\Delta_K^a}_{L^1}^{\vartheta}
    \norm{\Delta_K^a}_{L^{p_\star}}^{1-\vartheta}.
\] Taking the supremum over $a\in\mathcal A$ in
\eqref{eq:explicit-interpolation}, we obtain
\begin{align*}
    \sup_{a\in\mathcal A}
    \norm{\Delta_K^a}_{L^r}
    \le
    \left(
        \sup_{a\in\mathcal A}
        \E[\Delta_K^a]
    \right)^\vartheta \times
    \left(
        \sup_{a\in\mathcal A}
        \E\left[
            (\Delta_K^a)^{p_\star}
        \right]
    \right)^{(1-\vartheta)/p_\star}.
\end{align*}
The first factor converges to zero by
\eqref{eq:process_convergence}, while the second factor is uniformly
bounded by \eqref{eq:uniform-U-Delta-moments}. Hence
\begin{equation}
\label{eq:delta_interpolation}
    \sup_{a\in\mathcal A}
    \norm{\Delta_K^a}_{L^r}
    \longrightarrow0.
\end{equation}
We next estimate the other factor in
\eqref{eq:cost_error_interpolation}. Since $qs=p_\star$,
\begin{align*}
    \norm{1+(U_K^a)^q}_{L^s}^s
    &=
    \E\left[
        \left(
            1+(U_K^a)^q
        \right)^s
    \right]\\
    &\le
    2^{s-1}
    \left(
        1+\E\left[(U_K^a)^{qs}\right]
    \right)\\
    &=
    2^{s-1}
    \left(
        1+\E\left[(U_K^a)^{p_\star}\right]
    \right).
\end{align*}
Therefore, by \eqref{eq:uniform-U-Delta-moments},
\begin{equation}
\label{eq:U-weight-bound}
    \sup_{K\ge1}\sup_{a\in\mathcal A}
    \norm{1+(U_K^a)^q}_{L^s}
    <\infty.
\end{equation} Applying Hölder's inequality with the conjugate exponents $r$ and
$s$ in \eqref{eq:cost_error_interpolation}, and using
\eqref{eq:delta_interpolation} and \eqref{eq:U-weight-bound}, gives
\begin{align*}
    \sup_{a\in\mathcal A}
    \abs{J_K(a)-J(a)}
    &\le
    C_T
    \sup_{a\in\mathcal A}
    \norm{1+(U_K^a)^q}_{L^s}
    \sup_{a\in\mathcal A}
    \norm{\Delta_K^a}_{L^r}\\
    &\longrightarrow0.
\end{align*}
Thus, $\sup_{a\in\mathcal A}
    \abs{J_K(a)-J(a)}
    \longrightarrow0.$ Finally, define
$ \varepsilon_K
    :=
    \sup_{a\in\mathcal A}
    \abs{J_K(a)-J(a)}.
$
For every $a\in\mathcal A$,
\[
    J(a)-\varepsilon_K
    \le
    J_K(a)
    \le
    J(a)+\varepsilon_K.
\]
Taking the infimum over $a\in\mathcal A$ yields
\[
    V_0-\varepsilon_K
    \le
    V_0^K
    \le
    V_0+\varepsilon_K.
\]
Consequently,
\[
    \abs{V_0^K-V_0}
    \le
    \varepsilon_K
    =
    \sup_{a\in\mathcal A}
    \abs{J_K(a)-J(a)}
    \longrightarrow0.
\]
\end{proof}

}

\subsection{Online Markov-state update}\label{sec:onlinemarkov}

We now describe how the finite-dimensional Markov state can be
constructed online from observable data, in preparation for the
reinforcement-learning method developed in the next section. The decay
scale $\beta>0$ and the number of filters $K$ are user-chosen design
parameters, while $m$ is the number of observed event types. At each
decision time, the learner only needs the current physical state and the
times and component labels of the events observed since the preceding
decision time. These observations determine the exponential filters
through the recursion
\[
    Z^k_{t_{n+1}}
    =
    e^{-\beta k\Delta_n}Z^k_{t_n}
    +
    \sum_{(\tau_\ell,j_\ell)\in\mathcal E_n}
    e^{-\beta k(t_{n+1}-\tau_\ell)}e_{j_\ell},
    \qquad k=1,\ldots,K.
\]
Thus, the lifted state can be updated without knowing the kernel
coefficients $Q_k(\cdot)$, the baseline $\mu$, the intensity
$\lambda$, the original kernel $\Phi$, or the state-dynamics
coefficients $b,\sigma,\gamma$. The effect of the applied action on
the unknown dynamics is learned from the observed transitions and
rewards rather than computed from an explicit model.
Algorithm~\ref{alg:state_update} summarizes the resulting update.

\begin{algorithm}[H]
\caption{Online Markov-state update for current-action readout}
\label{alg:state_update}
\begin{algorithmic}[1]
\Require Decay scale $\beta>0$; number of filters $K$; Hawkes dimension $m$; grid $0=t_0<t_1<\cdots<t_N=T$.
\Require Observed state $X_{t_n}$ and event times/component labels $(\tau_\ell,j_\ell)$, where $j_\ell\in\{1,\dots,m\}$.
\Ensure Markov state $S_{t_n}^K=(t_n,X_{t_n},Z_{t_n}^1,\dots,Z_{t_n}^K)$.
\State Initialize $Z_0^k\gets 0\in\R^m$ for $k=1,\dots,K$.
\For{$n=0,1,\dots,N-1$}
    \State Form $S_{t_n}^K\gets (t_n,X_{t_n},Z_{t_n}^1,\dots,Z_{t_n}^K)$.
    \State Choose action $a_n$ using the actor, for example $a_n=\pi_\theta(S_{t_n}^K)+\varepsilon_n$.
    \State Apply $a_n$ on $[t_n,t_{n+1})$.
    \State Observe reward sample $R_n$, next state $X_{t_{n+1}}$, and events
    \[
        \mathcal E_n:=\{(\tau_\ell,j_\ell):t_n<\tau_\ell\le t_{n+1}\}.
    \]
    \State Set $\Delta_n\gets t_{n+1}-t_n$.
    \For{$k=1,\dots,K$}
        \State $Z_{t_{n+1}}^k\gets e^{-\beta k\Delta_n}Z_{t_n}^k$.
        \ForAll{$(\tau_\ell,j_\ell)\in\mathcal E_n$}
            \State $Z_{t_{n+1}}^k\gets Z_{t_{n+1}}^k+e^{-\beta k(t_{n+1}-\tau_\ell)}e_{j_\ell}$.
        \EndFor
    \EndFor
    \State Form $S_{t_{n+1}}^K\gets(t_{n+1},X_{t_{n+1}},Z_{t_{n+1}}^1,\dots,Z_{t_{n+1}}^K)$.
    \State Store $(S_{t_n}^K,a_n,R_n,S_{t_{n+1}}^K,\Delta_n)$ in the replay buffer.
\EndFor
\end{algorithmic}
\end{algorithm}

\begin{remark}[Binned observations]\label{rem:binned}
If exact event times inside $[t_n,t_{n+1}]$ are unavailable and only counts $\Delta N_n=N_{t_{n+1}}-N_{t_n}$ are observed, one may use the endpoint approximation
\[
    Z_{t_{n+1}}^k\approx e^{-\beta k\Delta_n}Z_{t_n}^k+\Delta N_n.
\]
The event-exact update in Algorithm~\ref{alg:state_update} is preferable whenever timestamps are available.
\end{remark}

\section{Controlled Hawkes and CT-DDPG}
\label{sec:ctddpg}
\subsection{Hawkes-Markov Decision Process and value function}
Throughout this section we fix an integer $K$. We consider a (time augmented) Markov Decision Process with state space given at any time $t$ by
\[
    Y_t^K=(t,X_t^{K,a},\mathbf Z^K_t)
    \in
\mathcal Y_K:=[0,T]\times \R^{d_x}\times(\R^m)^K,
\]
where $\mathbf Z^K_t:=Z_t^{K,1},\dots,Z_t^{K,K}$,
the random variable $X^K_t$ represents the observed value of $X^K$ solving \eqref{eq:Markov-poisson-embedded-hawkes-sde} at time $t$ while $Z^{K,i}$ are derived from the observation of the event time, see Section \ref{sec:onlinemarkov} and Remark \ref{rem:binned}. For the sake of simplicity, we remove the superscript $K$ in this section to alleviate the notations. 

We equip $(\mathbb R^m)^K$ with the product norm
\[
    \norm{\mathbf z^K}_{\mathcal Z_K}
    :=
    \sum_{k=1}^K\norm{z^k}_1.
\]
For each $i=1,\ldots,m$, define
\[
    \mathbf e_i^K
    :=
    (e_i,\ldots,e_i)
    \in(\mathbb R^m)^K,
\]
so that, at an event of type $i$, the memory vector changes from
$\mathbf z^K$ to $\mathbf z^K+\mathbf e_i^K$.

\begin{definition}[Deterministic policy and neural network parametrization]
    \label{def:policy}
    We denote by $\pi_\zeta:\mathcal Y_K\longrightarrow A$ a deterministic Markov policy as a neural network parametrized by the weight vector $\zeta\in \mathcal Z\subset \mathbb R^p$ for some $p\geq 1$. We denote by $\mathcal U$ the set of admissible policies $\pi^\zeta$  that Assumption \ref{ass:approx_uniform_moments} is satisfied by choosing $a=\pi^\zeta$ and 
    there exists a locally bounded function $\varpi_1 : [0, \infty) \to [0, \infty)$ such that for all $\zeta \in \mathcal Z$, $t \in [0, T]$, and $(t,x,\mathbf z^K),(t',x',{\mathbf z^K}') \in [0,T]\times\R^{d_x}\times(\R^m)^K$,
\[
    |\pi_\zeta(t,x,\mathbf z^K) - \pi_\zeta(t',x',{\mathbf z^K}')| \le \varpi_1(|\zeta|)\left(|t-t'|+\norm{x-x'}+\norm{\mathbf z-{\mathbf z^K}'}_{\mathcal Z_K}\right)\] and $|\pi_\zeta(t, 0,0)| \le \varpi_1(|\zeta|)$.

\end{definition}

We denote by  $\rho\ge0$ the reward discount rate. The fixed policy objective is given for any $t\in [0,T],\; \mathbf z^K=(z_1,\dots,z_K)\in (\mathbb R^m)^K$ by
\[
    J^\zeta_K(t,x,\mathbf z^K)
    :=
    \E_{t,x,\mathbf z^K;\pi^\zeta}\left[
        \int_t^T e^{-\rho (s-t)}c(s,X_s^K,\pi_\zeta(Y_s^K))\,\dd s
        +e^{-\rho (T-t)}g(X_T^K)
    \right].
\]
 The value function at time $t$ starting at the point $x,z$ is thus given by

\[V_K(t,x,\mathbf z^K)=\inf_{\pi^\zeta\in\mathcal U} J_K^\zeta(t,x,\mathbf z^K).\]
We recall the infinitesimal operator defined above by using the policy $\pi^\zeta$,
\begin{align*}
\label{eq:generator}
    \mathcal L_K^{\pi^\zeta(t,x,z)} f(t,x,z)
    &=
    \partial_t f(t,x,z)
    +
    b(t,x,\pi^\zeta(t,x,z))\cdot \nabla_x f(t,x,z)
    +
    \frac12\Tr\left[\sigma\sigma^\top(t,x,\pi^\zeta(t,x,z))\nabla_{xx}^2 f(t,x,z)\right]
    \nonumber\\
    &\quad
    -
    \sum_{k=1}^K \beta k\, z^k\cdot \nabla_{z^k}f(t,x,z)
    \nonumber\\
    &\quad
    +
    \sum_{i=1}^m
    \lambda_i^K(t,x,z,\pi^\zeta(t,x,z))
    \Big[
        f(t,x+\gamma_i(t,x,\pi^\zeta(t,x,z)),z^1+e_i,\dots,z^K+e_i)
        -f(t,x,z)
    \Big].
\end{align*}

\begin{theorem}
\label{thm:fixed_policy_bellman}
Fix $K\ge1$ and a deterministic Markov policy
$ \pi^\zeta\in\mathcal U$.
Suppose that
\[
    J_K^\zeta
    \in
    \mathcal C^{1,2,1}
    \left(
        [0,T]\times
        \mathbb R^{d_x}\times(\mathbb R^m)^K
    \right)
\]
where $\mathcal C^{1,2,1}$ means once continuously differentiable in time,
twice continuously differentiable in $x$, and once continuously
differentiable in each memory variable $\mathbf z^k$. Assume moreover that there exist $C>0$ and $\ell \in [1,p^*-2]$ such that,

\begin{align}
    &\norm{J_K^\zeta(u,x,\mathbf z^K)}
    +
    \norm{
        \partial_t
        J_K^\zeta(u,x,\mathbf z^K)
    }
    +
    \norm{\nabla_x J_K^\zeta(u,x,\mathbf z^K)} + \norm{\nabla_{\mathbf z^K} J_K^\zeta(u,x,\mathbf z^K)} \notag \\
    &+ \norm{\nabla_{xx} J_K^\zeta(u,x,\mathbf z^K)}\le
    C
    \left(
        1+\norm{x}^{\ell}
        +\norm{\mathbf z^K}_{\mathcal Z_K}^{\ell}
    \right),
    \label{eq:bellman_growth_condition}
\end{align}
 Then $J_K^\zeta$ satisfies
\begin{align}
    &
    \mathcal L_K^{
        \pi^\zeta(t,x,\mathbf z^K)
    }
    J_K^\zeta(t,x,\mathbf z^K)
    -
    \rho J_K^\zeta(t,x,\mathbf z^K)
    +
    c\left(
        t,x,
        \pi^\zeta(t,x,\mathbf z^K)
    \right)
    =
    0,
    \label{eq:fixed_policy_bellman_bold}
\end{align}
for $(t,x,\mathbf z^K)
    \in
    [0,T)\times
    \mathbb R^{d_x}\times(\mathbb R^m)^K,$
with terminal condition
\begin{equation}
\label{eq:fixed_policy_terminal_bold}
    J_K^\zeta(T,x,\mathbf z^K)
    =
    g(x).
\end{equation} Conversely, any function $v$ satisfying the same smoothness,
growth, and integrability conditions and solving
\eqref{eq:fixed_policy_bellman_bold}--%
\eqref{eq:fixed_policy_terminal_bold}
coincides with $J_K^\zeta$.
\end{theorem}

\begin{remark}\label{rem:LK}
   Note that \eqref{eq:bellman_growth_condition} together with Assumptions \ref{ass:standing}(S1)-(S2) and \ref{ass:kernel_approx_strong} there exists a constant $\tilde C>0$ such that $|\mathcal L_K^{
        \pi^\zeta(t,x,\mathbf z^K)
    }
    J_K^\zeta(t,x,\mathbf z^K)|\leq \tilde C(1+|x|^{p_\star}+\|\mathbf z^K\|^{p_\star}).$
\end{remark}

\begin{proof}[Proof of Theorem~\ref{thm:fixed_policy_bellman}]
Fix
$ (t,x,\mathbf z^K)
    \in
    [0,T]\times
    \mathbb R^{d_x}\times(\mathbb R^m)^K.$ For every $s\in[t,T]$, we have 
\begin{align}
    J_K^\zeta(t,x,\mathbf z^K)
    &=
    \E_{t,x,\mathbf z^K;\pi^\zeta}
    \Bigg[
        \int_t^s
        e^{-\rho(u-t)}
        c\left(
            Y_u^K,\pi^\zeta(Y_u^K)
        \right)
        \,\dd u
        +
        e^{-\rho(s-t)}
        J_K^\zeta
        \left(
           Y_s^K
        \right)
    \Bigg].
    \label{eq:fixed_policy_dpp_bold}
\end{align}
For $n\ge1$, define the local time
\[
    \tau_n
    :=
    \inf\left\{
        u\in[t,T]:
        \norm{X_u^K}
        +
        \norm{\mathbf Z_u^K}_{\mathcal Z_K}
        \ge n
    \right\}
    \wedge T.
\]
From the definition of $\mathcal U$, and by Assumption \ref{ass:approx_uniform_moments} we have
\[
    \sup_{t\le u\le T}
    \left(
        \norm{X_u^K}
        +
        \norm{\mathbf Z_u^K}_{\mathcal Z_K}
    \right)
    <\infty
    \qquad\mathbb P-\text{a.s.}
\]
Consequently, $\lim_n\tau_n=T$.
By the jump-diffusion It\^o formula on
$[t,s\wedge\tau_n]$ we get
we obtain
\begin{align*}
    &
    e^{-\rho(s\wedge\tau_n-t)}
    J_K^\zeta
    \left(
        s\wedge\tau_n,
        X_{s\wedge\tau_n}^K,
        \mathbf Z_{s\wedge\tau_n}^K
    \right)
    -
    J_K^\zeta(t,x,\mathbf z^K)+\int_t^{s\wedge\tau_n}
        e^{-\rho(u-t)}
        c\left(
            Y_u^K,\pi^\zeta(Y_u^K)
        \right)
        \,\dd u
    \notag\\
    &=
    \int_t^{s\wedge\tau_n}
    e^{-\rho(u-t)}
    \Bigg[
        \mathcal L_K^{a_u^\zeta}
        J_K^\zeta
        \left(
            u,X_{u-}^K,\mathbf Z_{u-}^K
        \right)
        -
        \rho
        J_K^\zeta
        \left(
            u,X_{u-}^K,\mathbf Z_{u-}^K
        \right) +
        c\left(
            Y_u^K,\pi^\zeta(Y_u^K)
        \right)
    \Bigg] \,\dd u\notag \\
    &\quad +
    \mathcal M_{s\wedge\tau_n}^{K},
\end{align*}
where
\begin{align*}
    &\mathcal M_{s\wedge\tau_n}^{K}
    :=
    \int_t^{s\wedge\tau_n}
    e^{-\rho(u-t)}
    \nabla_xJ_K^\zeta
    \left(
        u,X_{u-}^K,\mathbf Z_{u-}^K
    \right)^\top
    \sigma
    \left(
        u,X_{u-}^K,\pi^\zeta(Y_u^K)
    \right)
    \,\dd W_u
    \notag\\
    &
    +
    \sum_{i=1}^m
    \int_t^{s\wedge\tau_n}
    \!e^{-\rho(u-t)}
    \Bigg[
        J_K^\zeta
        \Big(
            u,
            X_{u-}^K
            +
            \gamma_i
            \left(
                u,X_{u-}^K,\pi^\zeta(Y_u^K)
            \right),
            \mathbf Z_{u-}^K\!+\!\mathbf e_i^K
        \Big)
        -
        J_K^\zeta
        \left(
           Y_u^K
        \right)
    \Bigg]
    \,\dd\widetilde M_u^{K,i},
\end{align*} 
with 
\begin{align*}
    \widetilde M_u^{K,i}
    &:=
    N_u^{K,i}-N_t^{K,i}
    -
    \int_t^u
    \lambda_i^K
    \left(
        s,X_{s-}^K,\mathbf Z_{s-}^K,\pi^\zeta(s,X_s^K,\mathbf Z_{s-}^K)
    \right)
    \,\dd s.
\end{align*}

Then, taking the expectation under condition \eqref{eq:bellman_growth_condition} together with \eqref{eq:fixed_policy_dpp_bold} we get for any $s\in [t,T]$

\[\mathbb E\left[\int_t^{s\wedge\tau_n}
    e^{-\rho(u-t)}
    \Bigg[
        \mathcal L_K^{a_u^\zeta}
        J_K^\zeta
        \left(
            u,X_{u-}^K,\mathbf Z_{u-}^K
        \right)
        -
        \rho
        J_K^\zeta
        \left(
            u,X_{u-}^K,\mathbf Z_{u-}^K
        \right) +
        c\left(
            Y_u^K,\pi^\zeta(Y_u^K)
        \right)
    \Bigg] \,\dd u\right]=0.\] 

By the dominated convergence theorem using Remark \ref{rem:LK} and Assumption \ref{ass:reward} we deduce that 

\[\mathcal L_K^{
        \pi^\zeta(t,x,\mathbf z^K)
    }
    J_K^\zeta(t,x,\mathbf z^K)
    -
    \rho J_K^\zeta(t,x,\mathbf z^K)
    +
    c\left(
        t,x,
        \pi^\zeta(t,x,\mathbf z^K)
    \right)
    =
    0.\]

Conversely, let a function $v$ satisfying 
\eqref{eq:fixed_policy_bellman_bold}--%
\eqref{eq:fixed_policy_terminal_bold}
and the same growth and integrability conditions as in \eqref{eq:bellman_growth_condition}. Applying
It\^o's formula to
\[
    e^{-\rho(u-t)}
    v\left(
        u,X_u^K,\mathbf Z_u^K
    \right)
\]
on $[t,\tau_n]$, and taking
expectations we get
\begin{align}
    v(t,x,\mathbf z^K)
    &=
    \E_{t,x,\mathbf z^K;\pi^\zeta}
    \Bigg[
        \int_t^{\tau_n}
        e^{-\rho(u-t)}
        c\left(
            u,X_u^K,a_u^\zeta
        \right)
        \,\dd u
        +
        e^{-\rho(\tau_n-t)}
        v\left(
            \tau_n,
            X_{\tau_n}^K,
            \mathbf Z_{\tau_n}^K
        \right)
    \Bigg].
    \label{eq:verification_stopped_bold}
\end{align} With $\tau_n\uparrow T$, using the polynomial-growth and the dominated convergence we get
\begin{align*}
    &\lim_{n\to\infty}
    \E
    \int_t^{\tau_n}
    e^{-\rho(u-t)}
    c\left(
        u,X_u^K,\pi^\zeta(Y_u^K)
    \right)
    \,\dd u
    =
    \E
    \int_t^T
    e^{-\rho(u-t)}
    c\left(
        u,X_u^K,\pi^\zeta(Y_u^K)
    \right)
    \,\dd u,
\end{align*}
and
\begin{align*}
    \lim_{n\to\infty}
    \E
    \left[
        e^{-\rho(\tau_n-t)}
        v\left(
            \tau_n,
            X_{\tau_n}^K,
            \mathbf Z_{\tau_n}^K
        \right)
    \right]
    &=
    \E
    \left[
        e^{-\rho(T-t)}
        v\left(
            T,X_T^K,\mathbf Z_T^K
        \right)
    \right]
    \\
    &
    =
    \E
    \left[
        e^{-\rho(T-t)}
        g\left(
            X_T^K
        \right)
    \right].
\end{align*}
Letting $n\to\infty$ in
\eqref{eq:verification_stopped_bold} therefore yields
$v(t,x,\mathbf z^K)
    =
    J_K^\zeta(t,x,\mathbf z^K).$
\end{proof}

\subsection{Advantage rate, performance and critic-advantage update}

\begin{definition}[Advantage rate]
\label{def:advantage}
The advantage-rate function associated with policy $\pi^\zeta$ for any action $a$ is given by
\[
    A_K^\zeta(t,x,\mathbf z^K,a)
    :=
    \mathcal L_K^a J_K^\zeta(t,x,\mathbf z^K)
    -\rho J_K^\zeta(t,x,\mathbf z^K)
    +c(t,x,a).
\]

\end{definition}

\begin{theorem}[Martingale characterization of the value and advantage rate]
\label{thm:martingale}
Fix $K\geq1$, $\zeta\in\mathcal Z$, and
$\pi^\zeta\in\mathcal U$. Let
    $\mathcal V
    \in
    \mathcal C^{1,2,1}
    \left(
        [0,T]\times
        \mathbb R^{d_x}\times(\mathbb R^m)^K
    \right)$ satisfy the same polynomial-growth and integrability conditions as in
Theorem~\ref{thm:fixed_policy_bellman}. Let $\mathcal Q:
    [0,T]\times\mathbb R^{d_x}
    \times(\mathbb R^m)^K\times A
    \longrightarrow\mathbb R$
be continuous with polynomial growth in $x,\mathbf z^K$ of order $p_\star$. 
Suppose further that
\begin{equation}
\label{eq:martingale_terminal_normalization}
    \mathcal V(T,x,\mathbf z^K)=g(x),
    \qquad
    \mathcal Q
    \left(
        t,x,\mathbf z^K,
        \pi^\zeta(t,x,\mathbf z^K)
    \right)
    =0.
\end{equation} For each $(t,x,\mathbf z^K)
    \in
    [0,T)\times
    \mathbb R^{d_x}\times(\mathbb R^m)^K,$
let $\mathcal O^\zeta_{t,x,\mathbf z^K}
    \subset A$
be a neighborhood of
$\pi^\zeta(t,x,\mathbf z^K)$. Assume that, for every $a\in\mathcal O^\zeta_{t,x,\mathbf z^K},$ there exists an admissible $A$-valued
$\mathbb F$-predictable process $\alpha=(\alpha_s)_{s\in[t,T]}$
such that
    $\lim_{s\downarrow t}\alpha_s=a,
    \mathbb P\text{-a.s.}$
Assume moreover that for $s\in[t,T]$,
\begin{align}
\label{eq:discounted_martingale_characterization}
    \mathbf M_s^{t,x,\mathbf z^K;\alpha}
    :=\;&
    e^{-\rho(s-t)}
    \mathcal V
    \left(
        s,X_s^{K,\alpha},\mathbf Z_s^{K,\alpha}
    \right)+
    \int_t^s
    e^{-\rho(u-t)}
    \Big[
        c\left(
            u,X_u^{K,\alpha},\alpha_u
        \right)
        -
        \mathcal Q
        \left(
            u,X_u^{K,\alpha},
            \mathbf Z_u^{K,\alpha},
            \alpha_u
        \right)
    \Big]
    \,\dd u,
\end{align}
is an $\mathbb F$-martingale. Then for every $
    (t,x,\mathbf z^K)
    \in
    [0,T]\times
    \mathbb R^{d_x}\times(\mathbb R^m)^K$ and $a\in\mathcal O^\zeta_{t,x,\mathbf z^K}$,\[
    \mathcal V(t,x,\mathbf z^K)
    =
    J_K^\zeta(t,x,\mathbf z^K),\quad
    \mathcal Q(t,x,\mathbf z^K,a)
    =
    A_K^\zeta(t,x,\mathbf z^K,a).
\]

\end{theorem}

\begin{proof}
For $(t,x,\mathbf z^K)
    \in
    [0,T)\times
    \mathbb R^{d_x}\times(\mathbb R^m)^K$
and $a\in\mathcal O^\zeta_{t,x,\mathbf z^K}$.
Let $\alpha$ be an exploratory control satisfying the assumptions
of the theorem. To simplify notation, write $X_s:=X_s^{K,\alpha},
    \mathbf Z_s:=\mathbf Z_s^{K,\alpha}.$
By the jump--diffusion It\^o formula, for every $s\in[t,T]$,
\begin{align}
\label{eq:martingale_characterization_ito}
    &
    e^{-\rho(s-t)}
    \mathcal V(s,X_s,\mathbf Z_s)
    -
    \mathcal V(t,x,\mathbf z^K)
    \notag\\
    &=
    \int_t^s
    e^{-\rho(u-t)}
    \Big[
        \mathcal L_K^{\alpha_u}
        \mathcal V
        \left(
            u,X_{u-},\mathbf Z_{u-}
        \right)
        -
        \rho
        \mathcal V
        \left(
            u,X_{u-},\mathbf Z_{u-}
        \right)
    \Big]
    \,\dd u
    +
    \mathcal N_s,
\end{align}
where $\mathcal N=(\mathcal N_s)_{s\in[t,T]}$ is the local martingale
\begin{align*}
    \mathcal N_s
    ={}&
    \int_t^s
    e^{-\rho(u-t)}
    \nabla_x\mathcal V
    \left(
        u,X_{u-},\mathbf Z_{u-}
    \right)^\top
    \sigma
    \left(
        u,X_{u-},\alpha_u
    \right)
    \,\dd W_u
    \\
    &+
    \sum_{i=1}^m
    \int_t^s
    e^{-\rho(u-t)}
    \left(\mathcal V
    \left(
        u,
        X_{u-}+\gamma_i(u,X_{u-},\alpha_u),
        \mathbf Z_{u-}+\mathbf e_i^K
    \right)-
    \mathcal V(u,X_{u-},\mathbf Z_{u-})\right)
    \,\dd\widetilde M_u^{K,\alpha,i},
\end{align*}
with
\[
    \widetilde M_s^{K,\alpha,i}
    :=
    N_s^{K,\alpha,i}
    -
    N_t^{K,\alpha,i}
    -
    \int_t^s
    \lambda_i^K
    \left(
        u,X_{u-},\mathbf Z_{u-},\alpha_u
    \right)
    \,\dd u.
\] Combining
\eqref{eq:martingale_characterization_ito} with
\eqref{eq:discounted_martingale_characterization} yields
\[
    \mathbf M_s^{t,x,\mathbf z^K;\alpha}
    -
    \mathbf M_t^{t,x,\mathbf z^K;\alpha}
    =
    \int_t^s
    e^{-\rho(u-t)}
    F_{\mathcal V,\mathcal Q}
    \left(
        u,X_{u-},\mathbf Z_{u-},\alpha_u
    \right)
    \,\dd u+\mathcal N_s,
\] where 
\begin{equation*}
    F_{\mathcal V,\mathcal Q}
    (t,x,\mathbf z^K,a):=
    \mathcal L_K^a
    \mathcal V(t,x,\mathbf z^K)
    -
    \rho\mathcal V(t,x,\mathbf z^K)
    +
    c(t,x,a)
    -
    \mathcal Q(t,x,\mathbf z^K,a).
\end{equation*}
By assumption,
$\mathbf M^{t,x,\mathbf z^K;\alpha}$ is a martingale, while
$\mathcal N$ is a local martingale. Therefore, $$\left(\int_t^s
    e^{-\rho(u-t)}
    F_{\mathcal V,\mathcal Q}
    \left(
        u,X_{u-},\mathbf Z_{u-},\alpha_u
    \right)
    \,\dd u\right)_{s\in[t,T]}$$ is a local martingale, which has the continuous paths with finite-variation. Hence, we have
\begin{equation}
\label{eq:zero_residual_integral}
    \int_t^s
    e^{-\rho(u-t)}
    F_{\mathcal V,\mathcal Q}
    \left(
        u,X_{u-},\mathbf Z_{u-},\alpha_u
    \right)
    \,\dd u
    =
    0,
    \qquad
    s\in[t,T],
    \quad
    \mathbb P\text{-a.s.}
\end{equation} Since the process is c\`adl\`ag,  $\displaystyle \lim_{u\downarrow t}
    \left(
        X_{u-},\mathbf Z_{u-}
    \right)
    =
    \left(
        x,\mathbf z^K
    \right),
    \mathbb P\text{-a.s.}$
Together with $\displaystyle  \lim_{u\downarrow t}\alpha_u= a,$ and the continuity of $F_{\mathcal V,\mathcal Q}$, we get
$$\begin{aligned}
    \lim_{u\downarrow t} e^{-\rho(u-t)}
    F_{\mathcal V,\mathcal Q}
    \left(
        u,X_{u-},\mathbf Z_{u-},\alpha_u
    \right)
    \longrightarrow
    F_{\mathcal V,\mathcal Q}
    (t,x,\mathbf z^K,a),\quad 
    \mathbb P\text{-a.s.}
\end{aligned}$$ Setting $s = t+h$ with $h>0$, dividing \eqref{eq:zero_residual_integral} by $h$, and taking
$h\downarrow0$, we obtain
\begin{align*}
    0
    &=
    \lim_{h\downarrow0}
    \frac1h
    \int_t^{t+h}
    e^{-\rho(u-t)}
    F_{\mathcal V,\mathcal Q}
    \left(
        u,X_{u-},\mathbf Z_{u-},\alpha_u
    \right)
    \,\dd u
    \\
    &=
    F_{\mathcal V,\mathcal Q}
    (t,x,\mathbf z^K,a).
\end{align*}
Consequently, for every
$a\in\mathcal O^\zeta_{t,x,\mathbf z^K}$,
\begin{equation}
\label{eq:q_equals_residual}
\begin{aligned}
    \mathcal Q(t,x,\mathbf z^K,a)
    ={}&
    \mathcal L_K^a
    \mathcal V(t,x,\mathbf z^K)
    -
    \rho\mathcal V(t,x,\mathbf z^K)+
    c(t,x,a).
\end{aligned}
\end{equation} Taking
    $a=\pi^\zeta(t,x,\mathbf z^K)$
in \eqref{eq:q_equals_residual} and using
\eqref{eq:martingale_terminal_normalization}, we obtain
\[
\begin{aligned}
    &
    \mathcal L_K^{
        \pi^\zeta(t,x,\mathbf z^K)
    }
    \mathcal V(t,x,\mathbf z^K)
    -
    \rho\mathcal V(t,x,\mathbf z^K)
    +
    c\left(
        t,x,\pi^\zeta(t,x,\mathbf z^K)
    \right)
    =
    0,
\end{aligned}
\]
with terminal condition
\[
    \mathcal V(T,x,\mathbf z^K)=g(x).
\]
Then Theorem~\ref{thm:fixed_policy_bellman} yields
\[
    \mathcal V(t,x,\mathbf z^K)
    =
    J_K^\zeta(t,x,\mathbf z^K).
\]
Substituting this identity into \eqref{eq:q_equals_residual} and using
Definition~\ref{def:advantage} gives
\[
    \mathcal Q(t,x,\mathbf z^K,a)
    =
    A_K^\zeta(t,x,\mathbf z^K,a).
\]
\end{proof}

\paragraph{Connection with critic-advantage learning.}

Theorem~\ref{thm:martingale} provides the identification principle
used to train the value and advantage-rate critics. For a fixed policy
$\pi^\zeta$, we approximate $J_K^\zeta$ by a value network parametrized by a vector of real numbers $\theta$
\[
    V_\theta:
    [0,T]\times\R^{d_x}\times(\R^m)^K
    \longrightarrow\R,
\]
and introduce an action-dependent network parametrized by a vector of real numbers $\psi$
\[
    \overline {\mathcal Q}_\psi:
    [0,T]\times\R^{d_x}\times(\R^m)^K\times A
    \longrightarrow\R.
\]
The corresponding normalized advantage-rate approximator is
\begin{equation}
\label{eq:normalized_advantage_critic}
    {\mathcal Q}_{\psi,\zeta}(t,x,\mathbf z^K,a)
    :=
    \overline {\mathcal Q}_\psi(t,x,\mathbf z^K,a)
    -
    \overline {\mathcal Q}_\psi
    \left(
        t,x,\mathbf z^K,
        \pi^\zeta(t,x,\mathbf z^K)
    \right).
\end{equation}
It satisfies
\[
    {\mathcal Q}_{\psi,\zeta}
    \left(
        t,x,\mathbf z^K,
        \pi^\zeta(t,x,\mathbf z^K)
    \right)
    =0,
\]
by construction, consistently with the normalization of
$A_K^\zeta$. For the exact pair
$(J_K^\zeta,A_K^\zeta)$, the discounted process in
Theorem~\ref{thm:martingale} is a martingale under locally exploratory
controls with candidates $\mathcal V=V_\theta$ and $\mathcal Q={\mathcal Q}_{\psi,\zeta}$. We therefore train $(V_\theta,{\mathcal Q}_{\psi,\zeta})$ by driving
discrete empirical increments of this process toward zero. The
resulting multi-step martingale residual and advantage-critic loss are defined
in Subsection~\ref{sec:hawkes_ctddpg_training}.

\subsection{Policy loss and actor update.}

From Theorem \ref{thm:fixed_policy_bellman} we recall that
\[
    A_K^\zeta(t,x,\mathbf z^K,\pi^\zeta(t,x,\mathbf z^K))=0.
\]

\begin{lemma}
\label{lem:performance_difference}
Let $\pi^\zeta,\pi^{\zeta'}$ be two deterministic Markov policies parameterized by $\zeta,\zeta'\in \mathcal Z$ and let the growth conditions in Theorem \ref{thm:fixed_policy_bellman} be satisfied for $J^\zeta_K$. Then 
\[
    J_K^{\zeta'}(t,x,\mathbf z^K)-J_K^\zeta(t,x,\mathbf z^K)
    =
    \E_{t,x,z;\pi^{\zeta'}}\left[
        \int_t^T e^{-\rho(u-t)}
        A_K^\zeta(Y_u^{\zeta'},\pi^{\zeta'}(Y_u^{\zeta'}))\,\dd u
    \right].
\]
\end{lemma}

\begin{proof}
 Using a similar localization in the proof of Theorem~\ref{thm:fixed_policy_bellman}, we get
\[
\begin{aligned}
    &\quad\E_{t,x,z;\pi^{\zeta'}}
    \left[
        e^{-\rho(T-t)}g(Y_T^{\zeta'})
    \right]
    -
    J_K^\zeta(t,x,\mathbf z^K)\\
    &=
    \E_{t,x,\mathbf z^K;\pi^{\zeta'}}\Big[
    \int_t^T e^{-\rho(u-t)}
    \left\{
        \mathcal L_K^{\pi^{\zeta'}(Y_u^{\zeta'})}J_K^\zeta(Y_u^{\zeta'})
        -
        \rho J_K^\zeta(Y_u^{\zeta'})
    \right\}\dd u\Big].
\end{aligned}\]
Therefore, by adding $
\E_{t,x,z;\pi^{\zeta'}}\Big[
    \int_t^T e^{-\rho(u-t)}c(Y_u^{\zeta'},\pi^{\zeta'}(Y_u^{\zeta'}))\,\dd u\Big]$
to both sides we get the equality. 
\end{proof}

\begin{assumption}
\label{ass:actor-advantage-regularity}
Fix $K \geq 1$ and $\zeta \in \mathcal{Z}$. There exist an open
neighborhood $\mathcal O_\zeta \subset \mathcal{Z}$ of $\zeta$, a constant
$C_{\zeta}>0$, and exponents $q_{\pi},q_A\geq 0$ such that  $q_{\pi}+q_A<p_\star$ and the following
conditions hold.

\begin{enumerate}
    \item 
   $
        (\zeta,y)\longmapsto \pi^{\zeta}(y)
$
    is continuous on $\mathcal O_{\zeta}\times\mathcal{Y}_K$; for every
    $y\in\mathcal{Y}_K$, the map
$
        \zeta\longmapsto \pi^{\zeta}(y)
$
    is continuously differentiable, and
 $  
        (\zeta,y)\longmapsto D_{\zeta}\pi^{\zeta}(y)
$
    is continuous on $\mathcal O_{\zeta}\times\mathcal{Y}_K$. Moreover,
$
        \pi^{\zeta}(y)\in\operatorname{int}(A)
$
    and
    \[
        \bigl\|D_{\zeta}\pi^{\zeta}(y)\bigr\|
        \leq
        C_{\zeta}\bigl(1+\|y\|^{q_{\pi}}\bigr),
        \qquad
        (\xi,y)\in \mathcal O_{\zeta}\times\mathcal{Y}_K.
    \]

    \item 
   
       $ Y^{\zeta_\varepsilon}\longrightarrow Y^{\zeta}$
        in $(\mathbb{P}\otimes dt)\text{-measure on }
        \Omega\times[0,T],
        \text{ for any }\zeta_\varepsilon\longrightarrow\zeta$ when $\varepsilon\longrightarrow 0$  

    \item 
    For every $y\in\mathcal{Y}_K$, the map
        $a\longmapsto A_K^{\zeta}(y,a)$
    admits a continuously differentiable extension to an open
    neighborhood of
      $  \bigl\{\pi^{\zeta}(y):\zeta\in \mathcal O_{\zeta}\bigr\}.
$ Moreover, 
    the map
$
        (y,a)\longmapsto\nabla_a A_K^{\zeta}(y,a)
$
    is continuous on the corresponding domain and 
    \[
        \left\|
        \nabla_a A_K^{\zeta}
        \bigl(y,\pi^{\zeta}(y)\bigr)
        \right\|
        \leq
        C_{\zeta}\bigl(1+\|y\|^{q_A}\bigr),
        \qquad
        (\zeta,y)\in \mathcal O_{\zeta}\times\mathcal{Y}_K.
    \]

\end{enumerate}

\end{assumption}

\begin{theorem}
\label{thm:dpg}
Under Assumption \ref{ass:actor-advantage-regularity}
\[
    \nabla_\zeta J_K^{\zeta}(x,\mathbf{z}^K)
    =
    \E\left[
        \int_0^T e^{-\rho t}
        D_\zeta\pi^\zeta(Y_t^\zeta)^\top
        \nabla_a A_K^\zeta(Y_t^\zeta,\pi^\zeta(Y_t^\zeta))\,\dd t
    \right].
\]
\end{theorem}

\begin{proof}
Fix $\zeta\in\mathcal{Z}$, $\varepsilon>0$ and a direction $\eta\in\R^r$ where $\varepsilon>0$ is sufficiently small so that $\zeta_\varepsilon:=\zeta+\varepsilon\eta$ is in $ \mathcal O_\zeta$.
From Lemma~\ref{lem:performance_difference}, we compute
\begin{equation}\label{eq:diffquotient}
    \frac{J_K^{\zeta_\varepsilon}(x,\mathbf{z}^K)-J_K^\zeta(x,\mathbf{z}^K)}{\varepsilon}
    =
    \E_{0,x,\mathbf{z}^K,;\pi^{\zeta_\varepsilon}}\Big[
    \int_0^T e^{-\rho t}
    \frac{
        A_K^\zeta(Y_t^{\zeta_\varepsilon},
        \pi_{\zeta_\varepsilon}(Y_t^{\zeta_\varepsilon}))
    }{\varepsilon}
    \,\dd t\Big].
\end{equation}
Recall that $A_K^\zeta(y,\pi_\zeta(y))=0$
holds for every state $y$. Therefore
\[
    A_K^\zeta(Y_t^{\zeta_\varepsilon},
    \pi_\zeta(Y_t^{\zeta_\varepsilon}))=0,
\]
and hence
\[
\begin{aligned}
\frac{
        A_K^\zeta(Y_t^{\zeta_\varepsilon},
        \pi_{\zeta_\varepsilon}(Y_t^{\zeta_\varepsilon}))
    }{\varepsilon}=
    \frac{
        A_K^\zeta(Y_t^{\zeta_\varepsilon},
        \pi_{\zeta_\varepsilon}(Y_t^{\zeta_\varepsilon}))
        -
        A_K^\zeta(Y_t^{\zeta_\varepsilon},
        \pi_\zeta(Y_t^{\zeta_\varepsilon}))
    }{\varepsilon}.
\end{aligned}
\]
Since $a\longmapsto A_K^\zeta(y,a)$ is continuously differentiable from Assumption \ref{ass:actor-advantage-regularity}, by the fundamental theorem of analysis with respect to the action variable, for any $y$, we have
\[
\begin{aligned}
    \frac{
        A_K^\zeta(y,\pi_{\zeta_\varepsilon}(y))
        -
        A_K^\zeta(y,\pi_\zeta(y))
    }{\varepsilon} = \int_0^1 
    \nabla_a A_K^\zeta(y, \pi_{\zeta+\varepsilon r \eta}(y))^\top \,\Big(
   \frac{
\pi^{\zeta_\varepsilon}(y)-\pi^\zeta(y)
}{\varepsilon}\Big) \; dr,
\end{aligned}
\]
where the derivative with respect to $a$ is understood with respect to the neural network parameterizing the action variable. Hence,

\begin{align*}
&
\frac{J^{\zeta_\varepsilon}_K(x,\mathbf{z}^K)-J^\zeta_K(x,\mathbf{z}^K)}{\varepsilon}\\
&=
\mathbb{E}_{0,x,\mathbf{z}^K;\pi^{\zeta_\varepsilon}}
\Bigg[
\int_0^T
e^{-\rho t}
\int_0^1
\nabla_aA^\zeta_K
\!\left(
Y_t^{\zeta_\varepsilon},
\pi^{\zeta+r\varepsilon\eta}
(Y_t^{\zeta_\varepsilon})
\right)^\top\times
\frac{
\pi^{\zeta_\varepsilon}(Y_t^{\zeta_\varepsilon})
-
\pi^\zeta(Y_t^{\zeta_\varepsilon})
}{\varepsilon}
\,dr\,dt
\Bigg].
\end{align*}

Recall from Assumption \ref{ass:actor-advantage-regularity} that
\[
\frac{
\pi^{\zeta_\varepsilon}(y)-\pi^\zeta(y)
}{\varepsilon}
\underset{\varepsilon\to 0}{\longrightarrow}
D_\zeta\pi^\zeta(y)\eta,\; \text{ with }\; \pi^{\zeta+r\varepsilon\eta}(y)
\underset{\varepsilon\to 0}{\longrightarrow}
\pi^\zeta(y)
\]
uniformly on compact sets. Moreover,
there exists some constant $C_\zeta'>0$ such that
\[\|\nabla_aA^\zeta_K
\!\left(
y,
\pi^{\zeta+r\varepsilon\eta}
(y)
\right)^\top\times
\frac{
\pi^{\zeta_\varepsilon}(y)
-
\pi^\zeta(y)
}{\varepsilon}\| \leq C'_\zeta(1+\|y\|^{q_\pi+q_A}).\]
Since $q_\pi+q_A<p^\star$, one can choose $\tilde q>1$ such that $\tilde q(q_\pi+q_A)\leq p^\star$. Therefore,

\[\sup_{\varepsilon>0, \zeta_\varepsilon\in \mathcal O_\zeta,\; r\in [0,1]} \mathbb E\Big[ \int_0^T \|\nabla_aA^\zeta_K
\!\left(
Y_t^{\zeta_\varepsilon},
\pi^{\zeta+r\varepsilon\eta}
(Y_t^{\zeta_\varepsilon})
\right)^\top\times
\frac{
\pi^{\zeta_\varepsilon}(Y_t^{\zeta_\varepsilon})
-
\pi^\zeta(Y_t^{\zeta_\varepsilon})
}{\varepsilon}\|^{\tilde q}\Big]dt<\infty.\]
Then, the family $\Big(\nabla_aA^\zeta_K
\!\left(
Y_t^{\zeta_\varepsilon},
\pi^{\zeta+r\varepsilon\eta}
(Y_t^{\zeta_\varepsilon})
\right)^\top\times
\frac{
\pi^{\zeta_\varepsilon}(Y_t^{\zeta_\varepsilon})
-
\pi^\zeta(Y_t^{\zeta_\varepsilon})
}{\varepsilon}\Big)_{\varepsilon>0, \zeta_\varepsilon\in \mathcal O_\zeta,\; r\in [0,1]}$ is uniformly integrable. 
By Vitali's convergence theorem, taking the limit when $\varepsilon\longrightarrow 0$ in  \eqref{eq:diffquotient} we deduce that \[
\nabla_{\zeta}J_K^{\zeta}(x,\mathbf{z}^K)
=
\mathbb{E}_{0,x,\mathbf{z}^K;\pi^{\zeta}}
\left[
    \int_0^T
    e^{-\rho t}
    D_{\zeta}\pi^{\zeta}(Y_t^{\zeta})^{\top}
    \nabla_a A_K^{\zeta}
    \left(
        Y_t^{\zeta},
        \pi^{\zeta}(Y_t^{\zeta})
    \right)
    dt
\right].
\]
\end{proof}

\paragraph{Connection with actor learning.}

Theorem~\ref{thm:dpg} provides the theoretical basis for the actor
update. Once the advantage-rate critic has been trained locally around
the current policy action, we use
\[
    \nabla_a\overline {\mathcal Q}_\psi
    \left(
        t,x,\mathbf z^K,
        \pi^\zeta(t,x,\mathbf z^K)
    \right)
    \approx
    \nabla_a A_K^\zeta
    \left(
        t,x,\mathbf z^K,
        \pi^\zeta(t,x,\mathbf z^K)
    \right).
\]
From
\eqref{eq:normalized_advantage_critic} we get
\[
    \nabla_a {\mathcal Q}_{\psi,\zeta}(t,x,\mathbf z^K,a)
    =
    \nabla_a\overline {\mathcal Q}_\psi(t,x,\mathbf z^K,a).
\]
Substituting these results yields a sample-based approximation of
$\nabla_\zeta J_K^\zeta$ and the precise actor loss and its joint
implementation with the critic update are given in
Subsection~\ref{sec:hawkes_ctddpg_training}.

\subsection{Full Hawkes CT-DDPG Algorithm}
\label{sec:hawkes_ctddpg_training}
 The numerical implementation follows the CT-DDPG method of
\cite{ChengGuoZhang2026} extended to Hawkes processes with additional terms in the advantage rate function. We extend the algorithm to the lifted Hawkes state
$
    Y_n^K
    :=
    \left(
        t_n,X_{t_n}^K,\mathbf Z_{t_n}^K
    \right),
$
where the memory filters $\mathbf Z_{t_n}^K$ are updated from the
observed event times according to
Algorithm~\ref{alg:state_update}. We therefore record only the main
critic and actor updates. Let $V_\theta:
    \mathcal Y_K\longrightarrow\R$
be the value network and let
$\overline {\mathcal Q}_\psi:
    \mathcal Y_K\times A\longrightarrow\R$
be the raw action-dependent advantage-rate network. For a fixed actor
$\pi^\zeta$, let ${\mathcal Q}_{\psi,\zeta}(y,a)$ be the normalized advantage-rate critic defined in \eqref{eq:normalized_advantage_critic}
\begin{equation*}
    {\mathcal Q}_{\psi,\zeta}(y,a)
    :=
    \overline {\mathcal Q}_\psi(y,a)
    -
    \overline {\mathcal Q}_\psi
    \left(
        y,\pi^\zeta(y)
    \right),
\end{equation*}
with ${\mathcal Q}_{\psi,\zeta}
    \left(
        y,\pi^\zeta(y)
    \right)
    =0,$
consistently with the normalization of
$A_K^\zeta$. For simplicity, we consider a time-discretization $(t_k)_{k\geq 0}$ with $t_0=0$ and path $h>0$, that is $t_{n+1}-t_n=h.$
Given a contiguous \(L\)-step trajectory segment sampled from the
replay buffer, where \(L\ge1\) denotes the multi-step TD horizon,
define
\begin{align*}
    \delta_{n,L}^{\theta,\psi,\zeta}
    :={}&
    e^{-\rho Lh}\widehat V_{n+L}
    -
    V_\theta(Y_n^K)+
    h\sum_{\ell=0}^{L-1}
    e^{-\rho\ell h}
    \Big[
        c\left(
            t_{n+\ell},
            X_{t_{n+\ell}}^K,
            a_{n+\ell}
        \right)
        -
        {\mathcal Q}_{\psi,\zeta}
        \left(
            Y_{n+\ell}^K,
            a_{n+\ell}
        \right)
    \Big],
\end{align*}
where
\[
    \widehat V_{n+L}
    :=
    \begin{cases}
        g(X_T^K),
        & t_{n+L}=T,\\
        V_{\theta^-}(Y_{n+L}^K),
        & t_{n+L}<T,
    \end{cases}
\]
and $V_{\theta^-}$ denotes a target value network. If the replay
buffer stores the cost accumulated over an interval rather than a cost
rate, the corresponding interval cost replaces $h\,c(t_n,X_{t_n}^K,a_n)$. For a batch $\mathcal B$ of replayed path segments, the critic loss is
\begin{align*}
    \mathcal L_{\mathrm{critic}}(\theta,\psi)
    :={}&
    \frac{1}{|\mathcal B|}
    \sum_{n\in\mathcal B}
    \left|
        \delta_{n,L}^{\theta,\psi,\zeta}
    \right|^2+
    \lambda_T
    \frac{1}{|\mathcal B_T|}
    \sum_{b\in\mathcal B_T}
    \left|
        V_\theta
        \left(
            T,X_T^{K,b},\mathbf Z_T^{K,b}
        \right)
        -
        g(X_T^{K,b})
    \right|^2,
\end{align*}
where $\mathcal B_T$ is a batch of terminal states. The critic
parameters are updated by
\[
    \theta
    \leftarrow
    \theta
    -
    \eta_\theta
    \nabla_\theta
    \mathcal L_{\mathrm{critic}},
    \qquad
    \psi
    \leftarrow
    \psi
    -
    \eta_\psi
    \nabla_\psi
    \mathcal L_{\mathrm{critic}}.
\]
During this update, the actor and target-network parameters are held
fixed. On the other hand, Theorem~\ref{thm:dpg} motivates the actor loss
\begin{equation}
\label{eq:actor_loss}
    \mathcal L_{\mathrm{actor}}(\zeta)
    :=
    \frac{1}{|\mathcal B_A|}
    \sum_{b\in\mathcal B_A}
    e^{-\rho t_b}
    \overline {\mathcal Q}_\psi
    \left(
        Y_b^K,
        \pi^\zeta(Y_b^K)
    \right),
\end{equation}
where $\mathcal B_A$ is a batch of nonterminal lifted states. Indeed,
holding the sampled states and $\psi$ fixed,
\[
\begin{aligned}
    \nabla_\zeta
    \mathcal L_{\mathrm{actor}}(\zeta)
    =
    \frac{1}{|\mathcal B_A|}
    \sum_{b\in\mathcal B_A}
    e^{-\rho t_b}
    D_\zeta\pi^\zeta(Y_b^K)^\top
    \nabla_a\overline {\mathcal Q}_\psi
    \left(
        Y_b^K,
        \pi^\zeta(Y_b^K)
    \right).
\end{aligned}
\]
The raw network $\overline {\mathcal Q}_\psi$ is used in
\eqref{eq:actor_loss}, since fully differentiating the normalized
quantity
${\mathcal Q}_{\psi,\zeta}(y,\pi^\zeta(y))$ would give zero. Because the present
problem is formulated as cost minimization, the actor update is
\[
    \zeta
    \leftarrow
    \zeta
    -
    \eta_\zeta
    \nabla_\zeta
    \mathcal L_{\mathrm{actor}}(\zeta).
\]

Replay sampling, exploration noise, target networks, and Polyak
averaging are implemented as in
\cite{ChengGuoZhang2026}. The complete procedure is summarized in
Algorithm~\ref{alg:hawkes_ctddpg}.

\begin{algorithm}[H]
\caption{Hawkes CT-DDPG}
\label{alg:hawkes_ctddpg}
\begin{algorithmic}[1]

\Require Actor $\pi^\zeta$, value critic $V_\theta$, raw advantage
critic $\overline {\mathcal Q}_\psi$, target networks, replay buffer
$\mathcal D$, and filter parameters $(\beta,K)$.

\For{each episode}

    \State Initialize $X_0^K$ and $\mathbf Z_0^K=0$.

    \For{$n=0,\ldots,N-1$}

        \State Form
        \[
            Y_{t_n}^K
            =
            (t_n,X_{t_n}^K,\mathbf Z_{t_n}^K).
        \]

        \State Select an exploratory action
        \[
            a_n
            =
            \operatorname{Proj}_A
            \bigl(
                \pi^\zeta(Y_{t_n}^K)+\varepsilon_n
            \bigr).
        \]

        \State Apply $a_n$, and observe the interval cost,
        $X_{t_{n+1}}^K$, and the events occurring on
        $(t_n,t_{n+1}]$.

        \State Update $\mathbf Z_{t_{n+1}}^K$ using
        Algorithm~\ref{alg:state_update}, form $Y_{n+1}^K$, and
        store the transition in $\mathcal D$.

        \State Sample replayed path segments and update
        $(\theta,\psi)$ by minimizing
        $\mathcal L_{\mathrm{critic}}$.

        \State Sample nonterminal lifted states and update
        $\zeta$ by minimizing
        $\mathcal L_{\mathrm{actor}}$.

        \State Update the target networks by Polyak averaging.

    \EndFor

\EndFor

\end{algorithmic}
\end{algorithm}

\section{Numerical illustrations}
\label{sec:numerics}
We now illustrate our result by comparing the Hawkes CT-DDPG efficiency in three settings for the Hawkes kernels. For the sake of simplicity, we work in the one-dimensional setting, that is $m=d_x=d_W=1$. We thus set a generative model to test the efficiency of our algorithm. The dynamics of the outcome $X$ is given by \[  dX_t = \bigl(b_0 - \kappa X_t - b_a a_t\bigr)\,dt
   + \bigl[\sigma_0 + \sigma_x X_t + \sigma_a a_t\bigr]_+ dW_t
   + \operatorname{clip}\!\bigl(\gamma_0 + \gamma_x X_{t-} + \gamma_a a_t,
   \gamma_{\min},\gamma_{\max}\bigr) dN_t,\]
where $N$ is a Hawkes process with intensity \[
 \lambda_t = \mu(t,X_{t-},a_t)
 + \alpha\,Q(a_t)\int_0^{t-}\phi(t-s)\,dN_s,\]
 with \[\mu(t,X,a) = \operatorname{clip}\!\bigl(\mu_0 + \mu_x X + \mu_a a,
 \mu_{\min},\mu_{\max}\bigr),
\text{ and }
 Q(a) = 1 - \frac{c_{\mathrm{eff}}\,a}{a_{\mathrm{half}} + a}.
\]
$Q$ represents a saturating Hill-type attenuation kernel. The control $a$ reduces the excitation gain, with half-effect at $a=a_{\mathrm{half}}$, and maximal reduction parameter $c_{\mathrm{eff}}$. Regarding the stochastic control optimization parameters, we choose \[
 c_t = c_x X_t^2 + c_a a_t^2,
 \qquad
 g(X_T) = \,\tilde c X_T^2 .
\] The numerical setup with choices of values for each parameter is given in Appendix \ref{app:setup}.
We recall that these dynamics and all the introduced parameters remain unknown for the Hawkes CT-DDPG method and are only set for data generation purposes and comparison with model-dependent methods (named the Oracle when Markovian solutions are available). The three scenarios tested are described below.
\begin{itemize}
    \item Single-exponential kernel. In this case, $\phi^{\exp}(\tau)=e^{-\rho_{\mathrm{exp}}\tau},$ with $\rho_{\mathrm{exp}}>0$, and the system is Markovian so that our Markovianization procedure is exact. We compare Hawkes CT-DDPG with an Oracle reducing the stochastic control problem to an HJB equation with state variable $X,\lambda$ similarly to \cite{Bensoussan2024,callegaro2025stochastic}.
    \item Erlang Kernel. We consider $\phi^{\mathrm{Erlang}}(\tau)=\rho_{\mathrm{er}}\tau e^{-\rho_\mathrm{er}\tau},$ with $\rho_\mathrm{er}>0.$
    In particular, the system can still be Markovianized introducing two auxiliary processes $ L^1, L^2$ with dynamics

     \[dL^1_t = -\rho_\mathrm{er}\, L^1_t\,dt + dN_t,\quad  dL^2_t = \rho_\mathrm{er}\bigl(L^1_t - L^2_t\bigr)\,dt.\]
Then,
\[ \lambda^{a}_t = \mu(t,X_{t-}^a,a_t) + \alpha\,Q(a_t)\,L^2_{t-}.\]
Here \(\alpha\) is the excitation-scale parameter in the general intensity
specification above; it is not generated by the Erlang Markovianization.
We refer to  \cite{duarte2019stability} for more details. 
\item Power-law kernel. This fully non-Markovian kernel is given by $\phi(\tau) = h\eta^h(\tau+\eta)^{-(1+h)},$
where $\tau, \eta, h>0.$ This kernel has infinite memory and is not Markovian.
\end{itemize} \paragraph{Known-parameter HJB benchmarks.}

To separate the error arising from model-free learning from the
difficulty of the underlying control problem, we compare the learned
policies with model-based benchmarks that are given full knowledge of
the environment parameters, including the coefficients
\(b,\sigma,\gamma,\mu\), the Hawkes kernel and its control dependence,
and the cost functions. For the single-exponential kernel, the Hawkes
memory admits an exact one-factor Markov representation. For the Erlang
kernel, it admits an exact two-phase Markov representation with
auxiliary states \((L^1,L^2)\). In these two experiments, we formulate
the corresponding finite-dimensional HJB equation and solve it
numerically using the Deep Galerkin Method introduced in \cite{Sirignano_2018}. We refer to the
resulting feedback policies as the known-parameter DGM/HJB oracles,
up to the numerical approximation error of the HJB solver. The power-law kernel does not admit an exact finite-dimensional Markov
representation. In this case, we first approximate the true kernel by
a finite exponential mixture,
\[
    \varphi^{\mathrm{mix}}(\tau)
    =
    \sum_{j=1}^{M}
    w_j e^{-\beta_j\tau},
\]
construct the corresponding finite-dimensional Markov state, and solve
the HJB equation of this approximating model using DGM. Since this
benchmark contains both kernel-approximation error and numerical HJB
error, it is referred to below as the \emph{mixture DGM/HJB benchmark},
rather than an exact oracle for the original power-law problem. All
benchmark and learned policies are evaluated using the same Monte
Carlo testing protocol.

All these kernels are tested with Hawkes CT-DDPG combining Algorithm \ref{alg:state_update} (Markovianization) and Algorithm \ref{alg:hawkes_ctddpg} (Hawkes CT-DDPG). We compare our results with an Oracle, the best static strategy and two other reinforcement learning methods in discrete-time: Soft Actor-Critic (SAC), see \cite{haarnoja2018soft} and Deep Deterministic Policy Gradient (DDPG), see \cite{silver2014deterministic}. Finally, we also emphasize the effect of the Markovianization Algorithm \ref{alg:state_update}. We compare the Hawkes CT-DDPG algorithm without taking into account the memory (current state, Algorithm \ref{alg:hawkes_ctddpg} only) with the Hawkes CT-DDPG Markovianized (Algorithms \ref{alg:state_update} and \ref{alg:hawkes_ctddpg}, named filtered) for Erlang and power-law kernels. The exact numerical values for each experiment with confidence intervals are given in Appendix \ref{app:setup}.

\paragraph{Single-exponential Figure \ref{fig:single}: sanity check with Oracle comparison.}
The true exponential kernel is represented exactly by one member of the
preselected filter. We note that CT-DDPG attains cost $0.1744$, reducing cost
by $46.3\%$ relative to the validation-selected static control $a\equiv0.39$,
whose cost is $0.3247$. The known-parameter DGM/HJB oracle, constructed from the exact
one-factor Markov representation, remains better at
$0.1658$, as expected. This experiment
therefore checks that the observed-filter implementation can exploit an exact
finite-dimensional memory state; it is not a comparison against SAC or
standard DDPG, and it does not contain a current-state ablation.

\begin{figure}
    \centering
    \includegraphics[width=0.5\linewidth]{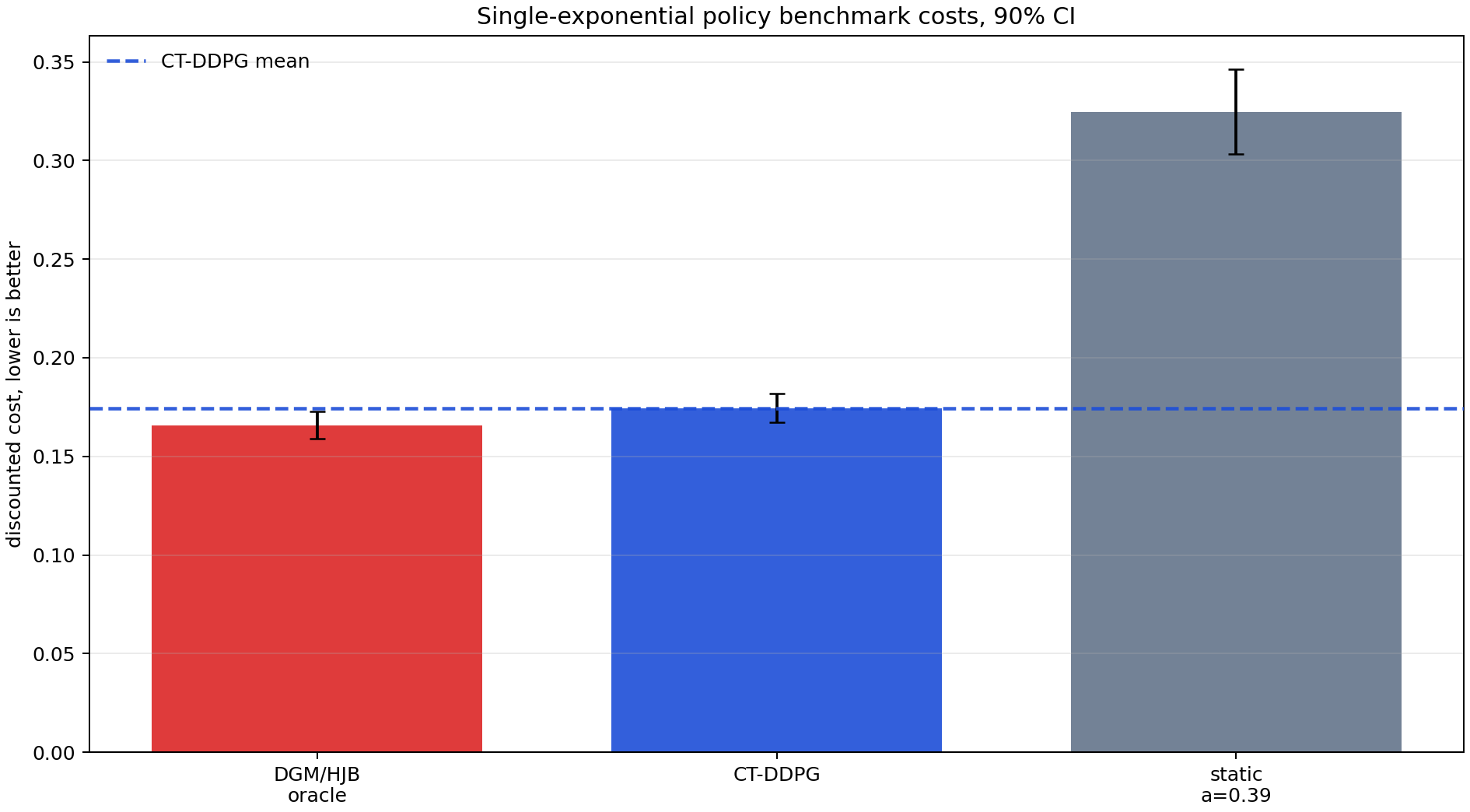}
    \caption{Single Kernel Environment:  Costs Comparison and Sanity Check. Exact-Oracle (red, left), CT-DDPG (blue, center), and validation-selected static strategy (grey, right). }
    \label{fig:single}
\end{figure}

\paragraph{Erlang kernel Figure \ref{fig:erlang}: Oracle \textit{vs.} Hawkes CT-DDPG \textit{vs} discrete time reinforcement learning.}
The oracle exploits the two-phase Markovianization \cite{duarte2019stability} reducing the problem to a standard stochastic control optimization with state $(X,L^1,L^2)$, while our Markovianization in Algorithm \ref{alg:hawkes_ctddpg} is applied to a 12-dimensional
exponential filter bank. The
DGM/HJB oracle remains as expected the best overall policy, with cost $0.1037$ and is set as benchmark for the efficiencies of Hawkes CT-DDPG. The filtered Hawkes CT-DDPG (with cost 0.1056) performs better than static, SAC and DDPG and remains the closest to the oracle. Relative to current-state observations, filtering
reduces mean cost by $13.4\%$ for CT-DDPG; $2.55\%$ for SAC, and $12.9\%$
for DDPG. Filtered CT-DDPG is numerically the best learned policy, improving on
filtered SAC and filtered DDPG by $0.89\%$ and $4.31\%$, respectively. 

\begin{figure}
    \centering
    \begin{subfigure}[b]{0.7\textwidth}
        \includegraphics[width=\linewidth]{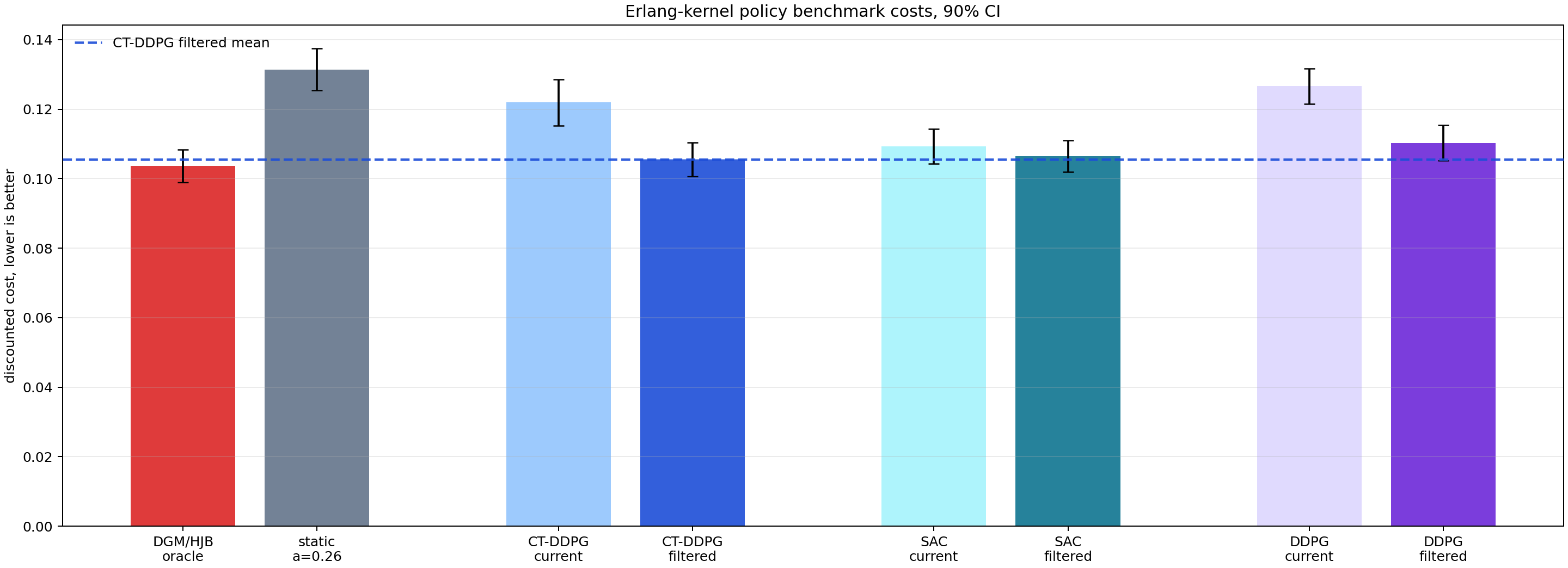}
    \end{subfigure}
    \begin{subfigure}[b]{0.7\textwidth}
        \includegraphics[width=\linewidth]{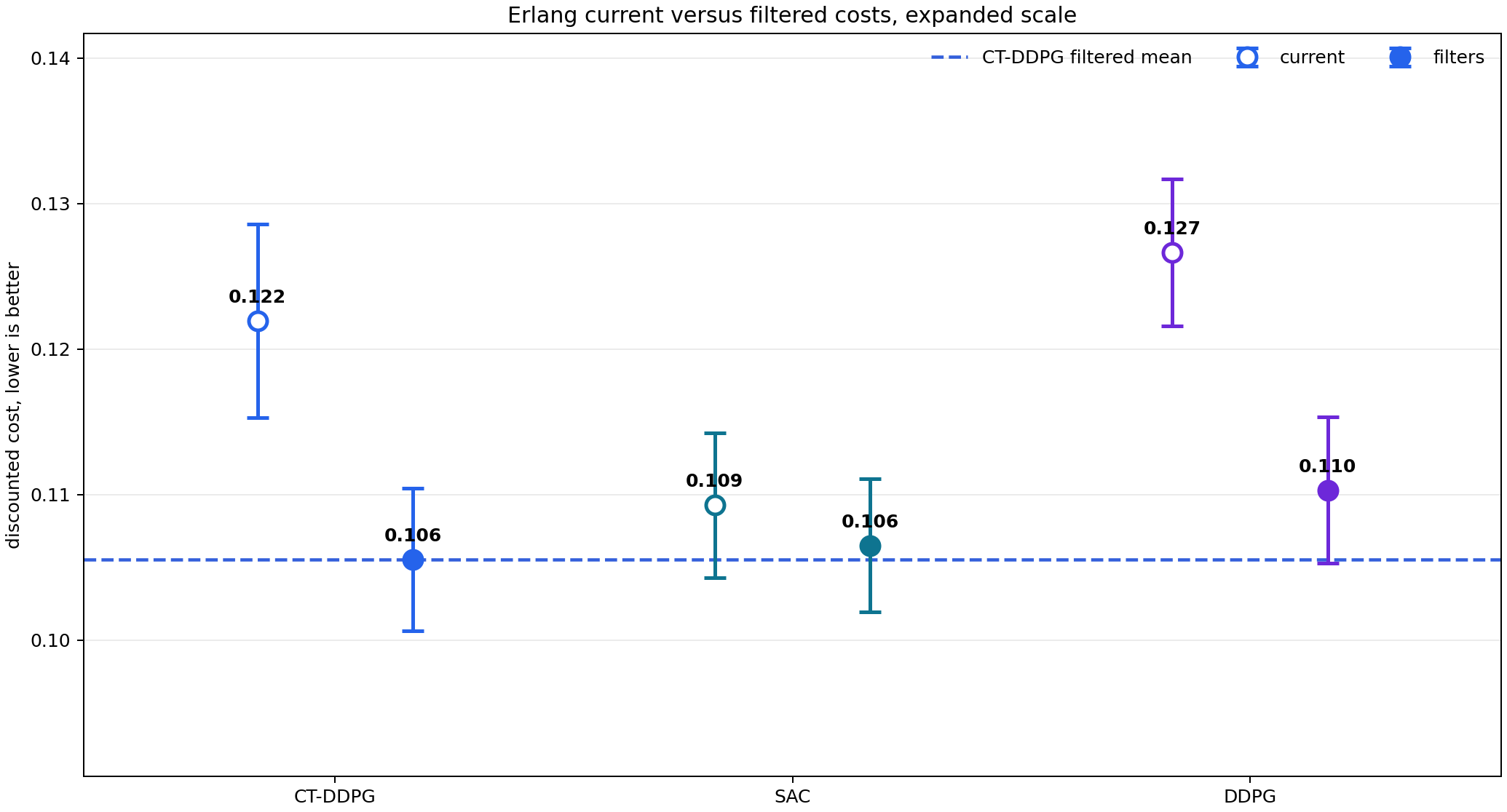}
    \end{subfigure}
    \caption{Erlang Kernel Environment Costs Comparison. Top from left to right: oracle (red), best static, CT-DDPG no Markovianization (current) and with Markovianization (filtered), SAC current and filtered and DDPG current and filtered. Bottom: value and confidence intervals for CT-DDPG current and filtered, SAC current and filtered, DDPG current and filtered.}
    \label{fig:erlang}
\end{figure}

\paragraph{Power-law kernel Figure \ref{fig:power}: Hawkes CT-DDPG \textit{vs.} discrete time reinforcement learning.}
Because the power-law model has no exact finite-dimensional Markov
representation, no exact HJB oracle is available. We instead report the
mixture DGM/HJB benchmark obtained by first approximating the
power-law kernel with a finite exponential mixture and then solving the
corresponding finite-dimensional HJB equation by DGM. The learning methods observe a 20-dimensional exponential filter approximation. Filtering reduces
mean cost by $4.67\%$ for CT-DDPG, $4.82\%$ for SAC, and $1.47\%$ for
DDPG. Filtered CT-DDPG again has the lowest cost among learned policies. Its
cost is $6.96\%$ below filtered SAC and $8.85\%$ below filtered DDPG, while
the mixture DGM/HJB oracle remains lower at $1.5696$.

\begin{figure}
    \centering
    \begin{subfigure}[b]{0.7\textwidth}
        \includegraphics[width=\linewidth]{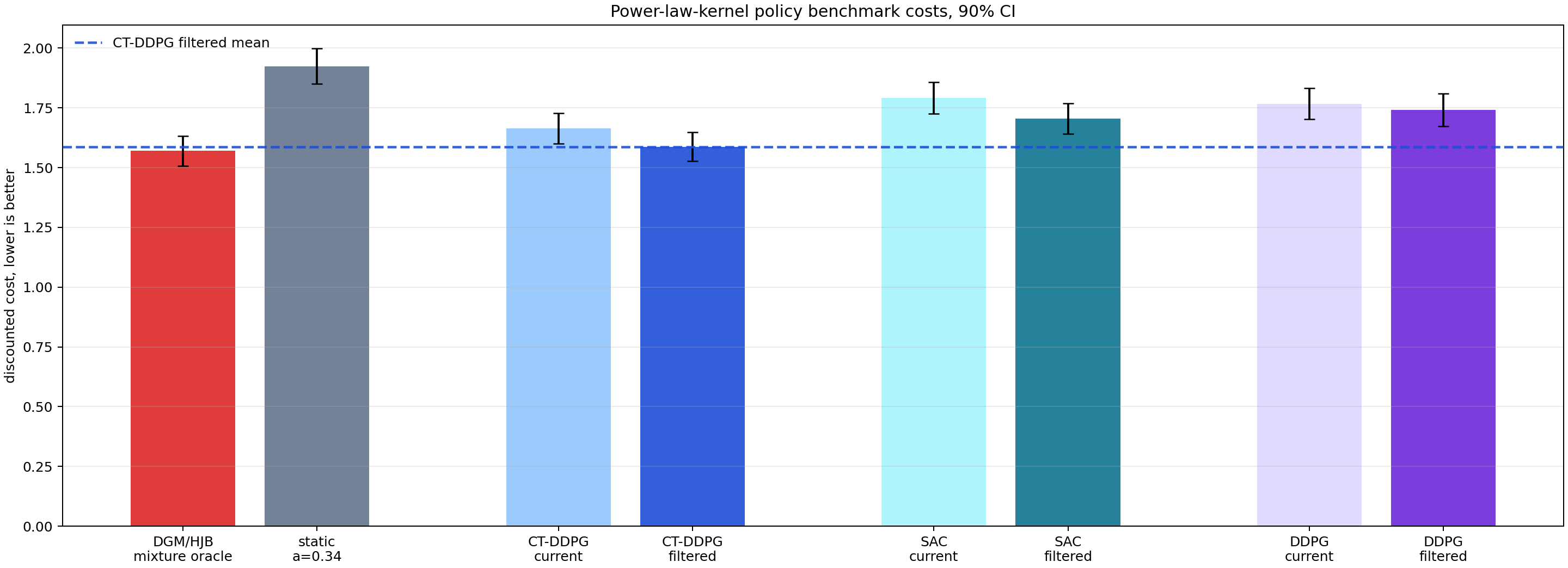}
    \end{subfigure}
    \begin{subfigure}[b]{0.7\textwidth}
        \includegraphics[width=\linewidth]{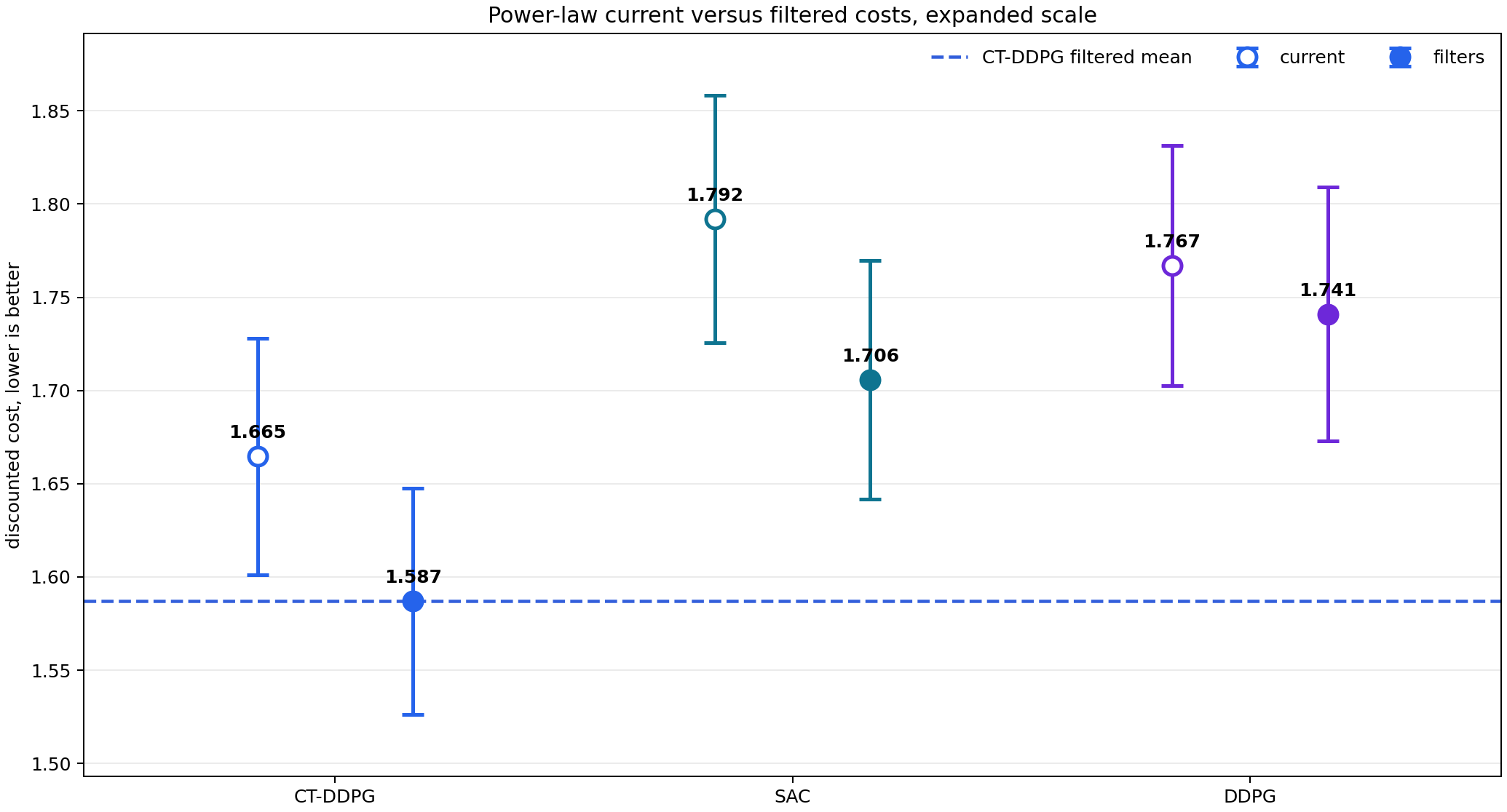}
    \end{subfigure}
    \caption{Power Kernel Environment Costs Comparison. Top from left to right: mixture oracle (red), best static, CT-DDPG no Markovianization (current) and with Markovianization (filtered), SAC current and filtered and DDPG current and filtered. Bottom: value and confidence intervals for CT-DDPG current and filtered, SAC current and filtered, DDPG current and filtered.}
    \label{fig:power}
\end{figure}

\paragraph{Cross-experiment conclusion.}
The single-exponential experiment serves as a sanity check, confirming that Hawkes CT-DDPG can effectively exploit a finite-dimensional Markov representation. The Erlang and power-law experiments provide more demanding tests: both models are non-Markovian in the physical state alone, although the Erlang kernel admits an exact finite-dimensional Markov lift, whereas the power-law kernel requires an exponential-mixture approximation. In both settings, augmenting the physical state with observable exponential filters improves the performance of CT-DDPG, SAC, and DDPG relative to their state-only counterparts. Moreover, when all model-free methods are supplied with the same preselected filter bank and filter-update rule, Hawkes CT-DDPG achieves the lowest mean cost in both experiments. These results provide consistent evidence that the proposed method is well suited to the control of Hawkes-driven systems with unknown dynamics and kernel parameters.

\bibliographystyle{plain}
				%
\bibliography{ref}

\appendix

\section{Numerical Setup}
\label{app:setup}

\newcolumntype{L}[1]{>{\raggedright\arraybackslash}p{#1}}

The simulator reports discounted reward internally, whereas the tables and
figures below report the equivalent positive discounted cost. Thus, lower
values indicate better policies. All Monte Carlo comparisons use common test
seeds within an experiment.

\begin{table}[h!]
\centering
\caption{High-level experiment definitions.}
\label{tab:experiment_definitions}
\small
\setlength{\tabcolsep}{3pt}
\renewcommand{\arraystretch}{1.12}
\begin{tabular}{@{}L{0.16\linewidth}L{0.26\linewidth}L{0.26\linewidth}L{0.26\linewidth}@{}}
\toprule
Item & Single exponential & Erlang kernel & Power-law kernel \\
\midrule
True kernel &
\(\phi(u)=e^{-1.30u}\) &
\(\phi(u)=\rho u e^{-\rho u}\), \(\rho=1.15\) (the scale is carried by \(\alpha\)) &
\(\phi(u)=b\eta^b(u+\eta)^{-(1+b)}\), \(\eta=0.12,\ b=0.80\) \\
True Markovization &
Exact, one exponential memory is included in the filter bank &
Exact, two Erlang auxiliary states \(Y^1,Y^2\) &
No finite exact Markovization; simulator uses event-history power-law memory \\
CT-DDPG memory state &
Exponential filter bank \(Z^k\), \(K=8\) &
Exponential filter bank \(Z^k\), \(K=12\) &
Exponential filter bank \(Z^k\), \(K=20\) \\
DGM/HJB oracle &
Known-parameter exact single-exponential Markov state &
Known-parameter exact Erlang Markov state &
Known-parameter exponential-mixture approximation with positive memory link \\
Generic RL baselines &
Not used in this sanity-check experiment &
SAC and DDPG with current-state and filtered observations &
SAC and DDPG with current-state and filtered observations \\
\bottomrule
\end{tabular}
\end{table}

\begingroup
\small
\setlength{\tabcolsep}{3pt}
\renewcommand{\arraystretch}{1.10}
\begin{longtable}{@{}L{0.19\linewidth}L{0.25\linewidth}L{0.25\linewidth}L{0.25\linewidth}@{}}
\caption{Environment and objective parameters.}
\label{tab:environment_parameters}\\
\toprule
Parameter & Single exponential & Erlang kernel & Power-law kernel \\
\midrule
\endfirsthead
\toprule
Parameter & Single exponential & Erlang kernel & Power-law kernel \\
\midrule
\endhead
\bottomrule
\endfoot
Horizon \(T\) & \(5.0\) & \(5.0\) & \(8.0\) \\
Time step \(\Delta t\) & \(0.02\) & \(0.02\) & \(0.02\) \\
Discount rate \(\beta_{\rm disc}\) & \(0.02\) & \(0.02\) & \(0.02\) \\
Baseline \(\mu(t,x,a)\) &
\(\operatorname{clip}(2.05+0.10x-0.08a,10^{-6},6)\) &
\(\operatorname{clip}(2.00+0.03x-0.04a,10^{-6},5)\) &
\(\operatorname{clip}(0.60+0.002x-0.002a,10^{-6},5)\) \\
Excitation scale \(\alpha\) & \(1.25\) & \(1.05\) & \(0.99\) \\
Control effect \(q(a)\) &
\(c_a=1.25,\ a_{1/2}=0.35\) &
\(c_a=1.30,\ a_{1/2}=0.35\) &
\(c_a=1.08,\ a_{1/2}=0.10\)\\
Filter decays &
\(0.325k,\ k=1,\ldots,8\) &
\(0.23k,\ k=1,\ldots,12\)&
\(1.00k,\ k=1,\ldots,20\) \\
Drift \(b(x,a)\) &
\(0.02-0.45x-0.72a\) &
\(0.02-0.45x-0.70a\) &
\(-1.00x-0.02a\) \\
Volatility \(\sigma(x,a)\) &
\([0.05+0.015x+0.08a]_+\) &
\([0.05+0.012x+0.08a]_+\) &
\([0.03+0.002x+0.005a]_+\) \\
Jump size \(\gamma(x,a)\) &
\(\operatorname{clip}(0.10+0.015x-0.015a,10^{-4},0.25)\) &
\(\operatorname{clip}(0.085+0.005x-0.010a,10^{-4},0.20)\) &
\(\operatorname{clip}(0.50+0.002x-0.05a,10^{-4},0.90)\) \\
State running-cost coefficient \(c_x\) & \(0.80\) & \(0.3000\) & \(1.00\) \\
Control running-cost coefficient \(c_a^{\rm cost}\) & \(0.18\) & \(0.2250\) & \(0.50\) \\
Terminal-cost coefficient \(c_T\) & \(0.60\) & \(0.2250\) & \(0.50\) 
\end{longtable}
\endgroup

For a transparent sufficient-condition check, let
\(\bar q=\sup_{a\in[0,1]}q(a)=1\). On \([0,T]\), the true excitation masses are
\[
\begin{aligned}
M_{\rm exp}
 &=1.25\int_0^5 e^{-1.30u}\,du=0.96009,\\
M_{\rm Erlang}
 &=1.05\int_0^5 1.15u e^{-1.15u}\,du=0.89343,\\
M_{\rm power}
 &=0.99\int_0^8 0.80(0.12)^{0.80}(u+0.12)^{-1.80}\,du=0.95601.
\end{aligned}
\]
All are strictly below one. The signed exponential approximations are checked
through their positive-link envelope, whose corresponding masses are also
below one as reported in the table. Together with compact controls, bounded
baseline and jump coefficients, finite horizon, and the displayed Lipschitz
coefficients, this gives sufficient conditions in Proposition~\ref{prop:sufficient_approx_uniform_moments}, which verify the
moment requirement in Assumption~\ref{ass:approx_uniform_moments} (with the numerical diagnostic using
\(p_*=4\)). Since the quadratic costs have growth index \(q=1\), this also
gives \(p_*>q+1=2\).

\begin{table}[h!]
\centering
\caption{Mean discounted cost with 90\% Monte Carlo intervals. Lower is better.
The intervals measure rollout uncertainty conditional on one trained policy;
bold values identify the best model-free learned policy in each comparison.}
\label{tab:numerical_results}
\small
\begin{tabular}{p{0.25\linewidth}ccc}
\toprule
Policy & Single exponential & Erlang kernel & Power-law kernel \\
\midrule
DGM/HJB oracle
& $0.1658\pm0.0070$
& $0.1037\pm0.0046$
& $1.5696\pm0.0617$ \\
CT-DDPG current
& \NA
& $0.1219\pm0.0067$
& $1.6647\pm0.0634$ \\
CT-DDPG filtered
& $0.1744\pm0.0073$
& $\mathbf{0.1056}\pm0.0049$
& $\mathbf{1.5870}\pm0.0607$ \\
SAC current
& \NA
& $0.1093\pm0.0050$
& $1.7920\pm0.0665$ \\
SAC filtered
& \NA
& $0.1065\pm0.0046$
& $1.7056\pm0.0641$ \\
DDPG current
& \NA
& $0.1266\pm0.0051$
& $1.7669\pm0.0644$ \\
DDPG filtered
& \NA
& $0.1103\pm0.0050$
& $1.7410\pm0.0680$ \\
Validation-selected static
& $0.3247\pm0.0214$
& $0.1314\pm0.0060$
& $1.9242\pm0.0734$ \\
\bottomrule
\end{tabular}
\end{table}

In the single-exponential column, CT-DDPG filtered denotes the proposed
CT-DDPG supplied with the eight-filter bank. Because the fourth decay equals
the true decay, this bank contains an exact Markov representation. SAC and
standard DDPG are not included in this sanity-check experiment.

\paragraph{Implementation and Reporting Notes.}
All neural policies are trained using normalized observations, bounded controls, minibatch optimization, gradient clipping, and a replay buffer following an initial data-collection period. Slowly updated target value networks are used for stability, with target actor networks additionally employed in the Erlang and power-law experiments, and actor updates are made more conservatively than critic updates when necessary. The reported actor skip terms are direct affine contributions from the current state and selected summaries of the observed filter bank to the actor logit; they provide a stable policy initialization but introduce neither additional observations nor knowledge of the true kernel. SAC and DDPG are trained with replay-based off-policy updates, vectorized simulation, exploration noise where appropriate, and exactly the same current-state or filtered observations used in the corresponding comparison. The DGM oracle is trained by sampling the HJB domain, with additional reachable-state sampling for the power-law approximation. During training, policies are periodically evaluated without exploration on held-out trajectories, and the checkpoint with the lowest validation cost is restored; validation episodes are used only for model selection and are separate from the final Monte Carlo sample. Warm starts, observation normalization, and reward rescaling are used only as numerical stabilization devices, while all reported results are evaluated on independent common-random trajectories and expressed in the original, unscaled cost units.

\end{document}